\documentclass[acmsmall]{acmart}

\AtBeginDocument{%
  \providecommand\BibTeX{{%
    \normalfont B\kern-0.5em{\scshape i\kern-0.25em b}\kern-0.8em\TeX}}}

\usepackage{url}            
\usepackage{booktabs}       
\usepackage{amsfonts}       
\usepackage{nicefrac}       
\usepackage{microtype}      
\usepackage{amssymb}
\usepackage{algorithm}
\usepackage{caption}
\usepackage{amsmath}
\usepackage{amsthm}
\usepackage{graphicx}
\usepackage{subfigure}
\usepackage{tabularx}
\usepackage{pifont}
\usepackage[noend]{algpseudocode}
\usepackage{bm}
\usepackage{array}
\usepackage{balance}
\usepackage{amsthm}
\usepackage{amsmath}
\usepackage{amssymb}
\usepackage{multirow}
\usepackage{fullpage}
\usepackage{color}

\allowdisplaybreaks

\newcommand{\Tr}{^{\rm \top}}

\newcommand{\Rcal}{{\mathcal{R}}}
\newcommand{\Acal}{{\mathcal{A}}}

\newcommand{\Fcal}{{\mathcal{F}}}
\renewcommand{\a}{{\bf a}}

\newcommand{\g}{{\bf g}}

\renewcommand{\u}{{\bf u}}
\renewcommand{\v}{{\bf v}}

\newcommand{\x}{{\bf x}}
\newcommand{\y}{{\bf y}}
\newcommand{\z}{{\bf z}}

\newcommand{\A}{{\bf A}}

\newcommand{\Lcal}{{\mathcal{L}}}

\renewcommand{\S}{{\bf S}}

\newcommand{\Ocal}[1]{{\mathcal{O}\left( #1  \right)}}
\newcommand{\Omegacal}[1]{{\Omega \left( #1 \right)}}

\newcommand{\Xcal}{{\mathcal{X}}}

\newcommand{\0}{{\bf 0}}

\newcommand{\argmin}{\operatornamewithlimits{argmin}}

\newcommand{\lrincir}[1]{\left( #1 \right)}

\newcommand{\lrnorm}[1]{\left\lVert#1\right\rVert}
\newcommand{\lrangle}[1]{\left\langle#1 \right\rangle}

\newcommand{\EE}{\mathop{\mathbb{E}}}
\newcommand{\RR}{\mathbb{R}}

\newcommand{\refabovecir}[2]{\displaystyle_{#1}^{#2}}

\newtheorem{Definition}{\bf{Definition}}
\newtheorem{Theorem}{\bf{Theorem}}
\newtheorem{Lemma}{\bf{Lemma}}

\newtheorem{Assumption}{\bf{Assumption}}

\newtheorem{Remark}{\bf{Remark}}

\begin{document}

\title{Understand Dynamic Regret with Switching Cost for Online Decision Making}

\author{Yawei Zhao}
\email{zhaoyawei@nudt.edu.cn}
\affiliation{%
  \institution{School of Computer, National University of Defense Technology}
  \streetaddress{109 Deya Road}
  \city{Changsha}
  \state{Hunan}
  \postcode{410073}
  \country{China}
}

\author{Qian Zhao}
\email{mcsqian.zhao@gmail.com}
\affiliation{%
  \institution{College of Mathematics and System Science, Xinjiang University}
  \city{Urumqi}
  \state{Xinjiang}
  \postcode{830001}
  \country{China}
}

\author{Xingxing Zhang}
\email{zhangxing@bjtu.edu.cn}
\affiliation{%
  \institution{Institute of Information Science, and Beijing Key Laboratory of Advanced Information Science and Network Technology, Beijing Jiaotong University}
  \city{Beijing}
  \postcode{100044}
  \country{China}
}

\author{En Zhu}
\email{enzhu@nudt.edy.cn}
\authornote{represents corresponding author.}
\author{Xinwang Liu}
\email{xinwangliu@nudt.edu.cn}
\affiliation{%
  \institution{School of Computer, National University of Defense Technology}
  \streetaddress{109 Deya Road}
  \city{Changsha}
  \state{Hunan}
  \postcode{410073}
  \country{China}
}

\author{Jianping Yin}
\email{jpyin@dgut.edu.cn}
\affiliation{%
  \institution{School of Computer, Dongguan University of Technology}
  \city{Dongguan}
  \state{Guangdong}
  \postcode{523808}
  \country{China}
}


\begin{abstract}
As a metric to measure the performance of an online method, dynamic regret with switching cost has drawn much attention for online decision making problems.  Although the sublinear regret has been provided in many previous researches, we still have little knowledge about the relation between the \textit{dynamic regret} and the \textit{switching cost}. In the paper, we investigate the relation for two classic online settings: Online Algorithms (OA) and Online Convex Optimization (OCO). We provide a new theoretical analysis framework, which shows an interesting observation, that is, the relation between the switching cost and the dynamic regret is different for settings of OA and OCO. Specifically, the switching cost has significant impact on the dynamic regret in the setting of OA.  But, it does not have an impact on the dynamic regret in the setting of OCO. Furthermore, we provide a lower bound of regret for the setting of OCO, which is same with the lower bound in the case of no switching cost. It shows that the switching cost does not change the difficulty of online decision making problems in the setting of OCO.
\end{abstract}

%

\keywords{Online decision making, dynamic regret, switching cost, online algorithms, online convex optimization, online mirror descent.}

\maketitle

\section{Introduction}
\label{sect_introduction}

Online Algorithms (OA)\footnote{Some literatures denote OA by `smoothed online convex optimization'.} \cite{pmlr-v75-chen18b,Chen:2016:UPO,renault-101007} and Online Convex Optimization (OCO) \cite{introduction-online-optimization,Hazan2016Introduction,ShalevShwartz:2012dz} are two important settings of online decision making. Methods in both OA and OCO settings are designed to  make a decision at every round, and then use the decision as a response to the environment. Their major difference is outlined as follows.
\begin{itemize}
\item For every round, methods in the setting of OA are able to know a loss function first, and then play a decision as the response to the environment.
\item However, for every round, methods in the setting of OCO have to play a decision before knowing the loss function. Thus, the environment may be adversarial to decisions of those methods.
\end{itemize}
Both of them have a large number of practical scenarios. For example, both the $k$-server problem \cite{Lee:2018wt,Bansal:2010:MTS} and the Metrical Task Systems (MTS) problem \cite{Abernethy:2010,Bubeck:2019vp,Bansal:2010:MTS}  are usually studied in the setting of OA. Other problems include online learning \cite{7401075,Yang:2013:EOL,8610120,8260919},  online recommendation \cite{Wang:2016:IAR}, online classification \cite{NIPS2003_2385,NIPS2010_3896}, online portfolio selection \cite{Li:2013:CWM}, and model predictive control \cite{MORARI1999667} are usually studied in the setting of OCO. 


Many recent researches begin to investigate performance of online methods in both OA and OCO settings by using \textit{dynamic regret with switching cost} \cite{pmlr-v75-chen18b,Li:2018uy}. It measures the difference between the cost yielded by real-time decisions and the cost yielded by the optimal decisions.   Comparing with the classic static regret \cite{introduction-online-optimization}, it has two major differences.
\begin{itemize}
\item First, it allows optimal decisions to change within a threshold over time, which is necessary in the dynamic environment\footnote{Generally, the dynamic environment means the distribution of the data stream may change over time.}.
\item Second, the cost yielded by a decision consists of two parts: the \textit{operating cost} and the \textit{switching cost}, while the classic static regret only contains the operating cost.
\end{itemize} The switching cost measures the difference between two successive decisions, which is needed in many practical scenarios such as service management in electric power network \cite{2b1edbffc}, dynamic resource management in data centers \cite{5934885,6269026,Wang:2014:ESG:2567529.2567556}. However, we still have little knowledge about the relation between the dynamic regret and the switching cost.  In the paper, we are motivated by the following fundamental questions.
\begin{itemize}
\item \textit{Does the switching cost impact the dynamic regret of an online method?}
\item \textit{Does the problem of online decision making  become more difficult due to the switching cost?}
\end{itemize}

To answer those challenging questions, we investigate online mirror descent in settings of OA and OCO, and provide a new theoretical analysis framework. According to our analysis, we find an interesting observation, that is, \textit{the switching cost does impact on the dynamic regret in the setting of OA. But, it has no impact on the dynamic regret in the setting of OCO.}  Specifically, when the switching cost is measured by $\lrnorm{\x_{t+1} - \x_t}^\sigma$ with $1\le \sigma \le 2$,  the dynamic regret for an OA method is $\Ocal{T^{\frac{1}{\sigma + 1}} D^{\frac{\sigma}{\sigma+1}}}$ where $T$ is the maximal number of rounds, and $D$ is the given budget of dynamics. But, the dynamic regret for an OCO method is $\Ocal{\sqrt{TD} + \sqrt{T}}$, which is same with the case of no switching cost \cite{Gyorgy:2016,Zhao:2018wx,Zinkevich:2003,Hall:2013vr}. Furthermore, we provide a lower bound of dynamic regret, namely $\Omegacal{\sqrt{TD} + \sqrt{T}}$ for the OCO setting. Since the lower bound is still same with the case of no switching cost \cite{Zhao:2018wx}, it implies that \textit{the switching cost does not change the difficulty of the online decision making problem for the OCO setting.} Comparing with previous results, our new analysis is more general than previous results. We define a new dynamic regret with a generalized switching cost, and provide new regret bounds. It is novel to analyze and provide the tight regret bound in the dynamic environment, since previous analysis cannot work directly for the generalized dynamic regert.  In a nutshell, our main contributions are summarized as follows.
\begin{itemize}
\item We propose a new general formulation of the dynamic regret with switching cost, and then develop a new analysis framework based on it.
\item We provide $\Ocal{T^{\frac{1}{\sigma + 1}} D^{\frac{\sigma}{\sigma+1}}}$ regret with $1\le \sigma\le 2$ for the setting of OA and $\Ocal{\sqrt{TD} + \sqrt{T}}$ regret for the setting of OCO by using the online mirror descent.
\item We provide a lower bound $\Omegacal{\sqrt{TD} + \sqrt{T}}$ regret for the setting of OCO, which matches with the upper bound.
\end{itemize}

The paper is organized as follows. Section \ref{sect_related_work} reviews related literatures. Section \ref{sect_preliminary} presents the preliminaries. Section \ref{sect_dynamic_regret_switching_cost_our_formulation} presents our new formulation of the dynamic regret with switching cost. Section \ref{sect_theoretical_analysis} presents a new analysis framework and main results. Section \ref{sect_empirical_study} presents extensive empirical studies. Section \ref{sect_conclusion} conludes the paper, and presents the future work.

\section{Related work}
\label{sect_related_work}
In the section, we review related literatures briefly.
\subsection{Competitive ratio and regret}
Although the competitive ratio is usually used to analyze OA methods, and the regret is used to analyze OCO methods, recent researches aim to developing unified frameworks to analyze the performance of an online method in both settings \cite{Blum2000,Abernethy:2010,Antoniadis-10.1007,pmlr-v23-buchbinder12,Bubeck:2018:KVM,Andrew:2013:TTM,Chen:2015:OCO}. \cite{Blum2000} provides an analysis framework, which is able to achieve sublinear regret for OA methods and constant competitive ratio for OCO methods. \cite{Abernethy:2010,pmlr-v23-buchbinder12,Bubeck:2018:KVM} uses a general OCO method, namely online mirror descent in the OA setting, and improves the existing competitive ratio analysis for $k$-server and MTS problems. Different from them, we extend the existing regret analysis framework to handle a general switching cost, and focus on investigating the relation between regret and switching cost. \cite{Antoniadis-10.1007} provides a lower bound for the OCO problem in the competitive ratio analysis framework, but we provide the lower bound in the regret analysis framework. \cite{Andrew:2013:TTM,Chen:2015:OCO} study the regret with switching cost in the OA setting, but the relation between them is not studied. Comparing with  \cite{Andrew:2013:TTM,Chen:2015:OCO}, we extend their analysis, and present a more generalized bound of dynamic regret (see Theorem \ref{theorem_regret_oa_upper_bound}).

\subsection{Dynamic regret and switching cost}

Regret is widely used as a metric to measure the performance of OCO methods. When the environment is static, e.g., the distribution of data stream does not change over time, online mirror descent yields $\Ocal{\sqrt{T}}$ regret for convex functions and  $\Ocal{\log T}$ regret for strongly convex functions \cite{introduction-online-optimization,Hazan2016Introduction,ShalevShwartz:2012dz}. When the distribution of data stream changes over time, online mirror descent yields $\Ocal{\sqrt{TD}+\sqrt{T}}$ regret for convex functions \cite{Gyorgy:2016}, where $D$ is the given budget of dynamics. Additionally, \cite{Zinkevich:2003} first investigates online gradient descent in the dynamic environment, and obtains $\Ocal{\sqrt{TD}+\sqrt{T}}$ regret (by setting $\eta \propto \sqrt{\frac{D}{T}}$) for convex $f_t$. Note that the dynamic regret used in \cite{Zinkevich:2003} does not contain swtiching cost. \cite{Hall:2013vr,Hall:2015ct} use similar but more general definitions of dynamic regret, and still achieves $\Ocal{\sqrt{TD}+\sqrt{T}}$ regret. Furthermore, \cite{Zhao:2018wx} presents that  the lower bound of the dynamic regret is $\Omegacal{\sqrt{TD}+\sqrt{T}}$.  Many other previous researches investigate the regret under different definitions of dynamics such as parameter variation \cite{Mokhtari:2016jz,Yang:2016ud,pmlr-v84-gao18a,Zhang:2016wl}, functional variation \cite{pmlr-v48-jenatton16,Besbes:2015gb,Zhang:2018tu}, gradient variation \cite{Chiang2012Online}, and the mixed regularity \cite{Jadbabaie:2015wg,Chen:2017tt}. Note that the dynamic regret in those previous studies does not contain switching cost, which is significantly different from our work. Our new analysis shows that this bound is achieved and optimal when there is switching cost in the regret (see Theorems \ref{theorem_regret_oco_upper_bound} and \ref{theorem_lower_bound_oco}). The proposed analysis framework thus shows how the switching cost impacts the dynamic regret for settings of OA and OCO, which leads to new insights to understand online decision making problems.

\section{Preliminaries}
\label{sect_preliminary}

\begin{table*}[!h]
\centering
\begin{tabular}{c|c|c|c|c}
\hline 
Algo. & Make decision first? & Observe $f_t$ first? & Metric & Has SC?\tabularnewline
\hline 
\hline 
OA & no & yes & competitive ratio & yes\tabularnewline
\hline 
OCO & yes & no & regret & no\tabularnewline
\hline 
\end{tabular}
\caption{Summary of difference between OA and OCO.  `SC' represents `switching cost'. }
\label{table_oa_oco_difference}
\end{table*}

In the section, we present the preliminaries of online algorithms and online convex optimization, and highlight their difference. Then, we present the dynamic regret with switching cost, which is used to measure the performance of both OA methods and OCO methods. 

\subsection{Online algorithms and online convex optimization}

Comparing with the setting of OCO \cite{ShalevShwartz:2012dz,Hazan2016Introduction,introduction-online-optimization},  OA has the following major difference. 
\begin{itemize}
\item OA assumes that the loss function, e.g., $f_t$, is known before making the decision at every round. But, OCO assumes that the loss function, e.g., $f_t$, is given after making the decision at every round.
\item The performance of an OA method is measured by using the \textit{competitive ratio} \cite{pmlr-v75-chen18b}, which is defined by
\begin{align}
\nonumber
\frac{ \left [\sum_{t=1}^T \lrincir{f_t(\x_t) +  \lrnorm{\x_t - \x_{t-1}}}\right ] }{ \left [\sum_{t=1}^T \lrincir{f_t(\x_t^\ast) +  \lrnorm{\x_t^\ast - \x_{t-1}^\ast}}\right ] }.
\end{align} Here, $\{\x_t^\ast\}_{t=1}^T$ is denoted by
\begin{align}
\nonumber
\{\x_t^\ast\}_{t=1}^T = \argmin_{\{\z_t\}_{t=1}^T \in \tilde{\Lcal}_D^T} \sum_{t=1}^T \lrincir{ f_t(\z_t) + \lrnorm{\z_t - \z_{t-1}} } 
\end{align} where $\tilde{\Lcal}_D^T := \left \{ \{\z_t\}_{t=1}^T : \sum_{t=1}^{T} \lrnorm{\z_t - \z_{t-1}} \le D   \right \}$. $D$ is the given budget of dynamics. It is the best offline strategy, which is yielded by knowing all the requests beforehand \cite{pmlr-v75-chen18b}. Note that $\lrnorm{\x_t^\ast - \x_{t-1}^\ast}$ is the switching cost yielded by $A$ at the $t$-th round. But, OCO is usually measured by the \textit{regret}, which is defined by
\begin{align}
\nonumber
\sum_{t=1}^T f_t(\x_t) - \min_{\{\z_t\}_{t=1}^T \in \Lcal_D^T}\sum_{t=1}^T f_t(\z_t),
\end{align} where $\Lcal_D^T := \left \{ \{\z_t\}_{t=1}^T : \sum_{t=1}^{T-1} \lrnorm{\z_{t+1} - \z_t} \le D   \right \}$. $D$ is also the given budget of dynamics. Note that the regret in classic OCO algorithm does not contain the switching cost. 
\end{itemize} To make it clear, we use Table \ref{table_oa_oco_difference} to highlight their differences.

\subsection{Dynamic regret with switching cost}

Although the analysis framework of OA and OCO is different, the \textit{dynamic regret with switching cost} is a popular metric to measure the performance of both OA and OCO \cite{pmlr-v75-chen18b,Li:2018uy}. Formally, for an algorithm $A$, its dynamic regret with switching cost $\widetilde{\Rcal}_D^A$ is defined by
\begin{align}
\label{equa_dynamic_regret_switching_cost_obd}
\widetilde{\Rcal}_D^A := & \sum_{t=1}^T f_t(\x_t) + \sum_{t=1}^{T-1}\lrnorm{\x_{t+1} - \x_t}  -  \min_{\{\z_t\}_{t=1}^T \in \Lcal_D^T}\lrincir{\sum_{t=1}^T f_t(\z_t) + \sum_{t=1}^{T-1}\lrnorm{\z_{t+1} - \z_t}},
\end{align} where $\Lcal_D^T := \left \{ \{\z_t\}_{t=1}^T : \sum_{t=1}^{T-1} \lrnorm{\z_{t+1} - \z_t} \le D   \right \}$. Here, $\lrnorm{\x_{t+1} - \x_t}$ represents the switching cost at the $t$-th round. $D$ is the given budget of dynamics in the dynamic environment. When $D=0$, all optimal decisions are same. With the increase of $D$, the optimal decisions are allowed to change to follow the dynamics in the environment. It is necessary when the distribution of data stream changes over time.

\subsection{Notations and Assumptions.} 

We use the following notations in the paper. 
\begin{itemize}
\item 
The bold lower-case letters, e.g., $\x$,  represent vectors.  The normal letters, e.g., $\mu$, represent a scalar number.    
\item $\lrnorm{\cdot}$ represents a general norm of a vector.  
\item $\Xcal^T$ represents Cartesian product, namely, $\underbrace{ \Xcal \times \Xcal \times ... \times \Xcal }_{T \text{ times}}$. $\Fcal^T$ has the similar meaning.
\item Bregman divergence $B_{\Phi}(\x,\y)$ is defined by $ B_\Phi(\x,\y) = \Phi(\x) - \Phi(\y) - \lrangle{\nabla \Phi(\y), \x - \y}$.  
\item $\Acal$ represents a set of all possible online methods, and $A\in \Acal$ represents some a specific online method. 
\item $\lesssim$ represents `less than equal up to a constant factor'.
\item $\EE$ represents the mathematical expectation operator.
\end{itemize}

Our assumptions are presented as follows. They are widely used in previous literatures \cite{Li:2018uy,pmlr-v75-chen18b,ShalevShwartz:2012dz,Hazan2016Introduction,introduction-online-optimization}.
\begin{Assumption}
\label{assumption_basic}
The following basic assumptions are used throughout the paper.
\begin{itemize}
\item For any $t\in[T]$, we assume that $f_t$ is convex, and has $L$-Lipschitz gradient.
\item The function $\Phi$ is $\mu$-strongly convex, that is, for any $\x\in\Xcal$ and $\y\in\Xcal$,  $ B_{\Phi}(\x,\y) \ge \frac{\mu}{2}\lrnorm{\x-\y}^2 $.
\item For any $\x\in\Xcal$ and $\y\in\Xcal$, there exists a positive constant $R$ such that  
\begin{align}
\nonumber
\max \left \{ B_\Phi(\x, \y), \lrnorm{\x - \y}^2 \right \}    \le R^2.
\end{align}
\item For any $\x\in\Xcal$, there exists a positive constant $G$ such that 
\begin{align}
\nonumber
\max \left \{ \lrnorm{\nabla f_t(\x)}^2, \lrnorm{\nabla \Phi(\x)}^2 \right \} \le G^2
\end{align}
\end{itemize}
 
\end{Assumption}

\section{Dynamic regret with generalized switching cost}
\label{sect_dynamic_regret_switching_cost_our_formulation}
In the section, we propose a new formulation of dynamic regret, which contains a generalized switching cost. Then, we highlight the novelty of this formulation, and present the online mirror decent method for setting of OA and OCO.

\subsection{Formulation}
\label{subsect_formulation}

For an algorithm $A\in \Acal$, it yields a cost at the end of every round, which consists of two parts: \textit{operating cost} and \textit{switching cost}. At the $t$-th round, the \textit{operating cost} is incurred by $f_t(\x_t)$, and the \textit{switching cost} is incurred by $\lrnorm{\x_{t+1} - \x_t}^{\sigma}$ with $1\le \sigma \le 2$. The optimal decisions are denoted by $\{\y_t^\ast\}_{t=1}^T$, which is denoted by
\begin{align}
\nonumber
\{\y_t^\ast\}_{t=1}^T = \argmin_{\{\y_t\}_{t=1}^T \in \Lcal_D^T } \sum_{t=1}^T f_t(\y_t) + \sum_{t=1}^{T-1}\lrnorm{\y_{t+1} - \y_t}^{\sigma}.
\end{align} Here, $\Lcal_D^T$ is denoted by
\begin{align}
\nonumber
\Lcal_D^T = \left\{ \{\y_t\}_{t=1}^T : \sum_{t=1}^{T-1}\lrnorm{\y_{t+1}-\y_t} \le D \right\}.
\end{align}
$D$ is a given budget of dynamics, which measures how much the optimal decision, i.e., $\y_t^{\ast}$ can change over $t$. With the increase of $D$, those optimal decisions can change over time to follow the dynamics in the environment effectively.

Denote an optimal method $A^\ast$, which yields the optimal sequence of decisions $\{\y_t^\ast\}_{t=1}^T$. Its total cost is denoted by 
\begin{align}
\nonumber
\text{cost}(A^{\ast}) = \sum_{t=1}^T f_t(\y_t^\ast) + \sum_{t=1}^{T-1} \lrnorm{\y_{t+1}^\ast - \y_t^\ast}^{\sigma}.
\end{align} Similarly, the total cost of an algorithm $A\in \Acal$ is denoted by 
\begin{align}
\nonumber
\text{cost}(A) = \sum_{t=1}^T f_t(\x_t) + \sum_{t=1}^{T-1}\lrnorm{\x_{t+1} - \x_t}^{\sigma}.
\end{align}

\begin{Definition}
For any algorithm $A\in\Acal$, its dynamic regret $\Rcal_D^A$ with switching cost  is defined by 
\begin{align}
\label{equa_our_definition_dynamic_regret_switching_cost}
\Rcal_D^A :=  \text{cost}(A) - \text{cost}(A^{\ast}).
\end{align}
\end{Definition} 
Our new formulation of the dynamic regret $\Rcal_D^A$ makes a balance between the operating cost and the switching cost, which is different from the previous definition of the dynamic regret in \cite{Zinkevich:2003,Gyorgy:2016,Hall:2013vr}.

Note that the freedom of $\sigma$ with $1 \le \sigma \le 2$ allows our new dynamic regret $\Rcal_D^A$ to measure the performance of online methods for a large number of problems. Some problems such as dynamic control of data centers \cite{Lin:2012:OOS}, stock portfolio management \cite{Li:2014:OPS},  require to be sensitive to  the small change between successive decisions, and the switching cost in these problems is usually bounded by $\lrnorm{\x_{t+1} - \x_t}$.  But, many problems such as dynamic placement of cloud service \cite{6258025:zhang} need to bound the large change between successive decisions effectively, and the switching cost in these problems is usually bounded by $\lrnorm{\x_{t+1} - \x_t}^2$.

\subsection{Novelty of the new formulation}
Our new formulation of the dynamic regret is more general than previous formulations \cite{pmlr-v75-chen18b,Li:2018uy}, which are presented as follows. 
\begin{itemize}
\item \textbf{Support more general switching cost.} \cite{pmlr-v75-chen18b} defines the dynamic regret  with switching cost  by \eqref{equa_dynamic_regret_switching_cost_obd}. It is a special case of our new formulation \eqref{equa_our_definition_dynamic_regret_switching_cost} by setting $\sigma = 1$.  The sequence of optimal decisions $\{\y_t^\ast\}_{t=1}^T$ is dominated by $\{f_t\}_{t=1}^T$ and $D$, and does not change over $\{\x_t\}_{t=1}^T$. $\Rcal_D^A$ is thus impacted by $\{\x_t\}_{t=1}^T$ for the given $\{f_t\}_{t=1}^T$ and $D$. Generally, $\lrnorm{\x_{t+1} - \x_t}$ is more sensitive to measure the slight change between $\x_{t+1}$ and $\x_t$ than $\lrnorm{\x_{t+1} - \x_t}^2$.  But,  for some problems such as the dynamic placement of cloud service \cite{6258025:zhang}, the switching cost at the $t$-th round is usually measured by $\lrnorm{\x_{t+1} - \x_t}^2$, instead of $\lrnorm{\x_{t+1} - \x_t}$. The previous formulation in \cite{pmlr-v75-chen18b} is not suitable to bound the switching cost for those problems. Benefiting from $1\le \sigma \le 2$, \eqref{equa_our_definition_dynamic_regret_switching_cost} supports more general switching cost than previous work. 
\item \textbf{Support more general convex $f_t$.} \cite{Li:2018uy} defines the  the dynamic regret with switching cost by 
\begin{align}
\nonumber
\sum_{t=1}^T f_t(\x_t) + \sum_{t=1}^{T-1}\lrnorm{\x_{t+1} - \x_t}^2 - \min_{\{\z_t\}_{t=1}^T \in \Xcal^T} \lrincir{\sum_{t=1}^T f_t(\z_t) + \sum_{t=1}^{T-1}\lrnorm{\z_{t+1} - \z_t}^2},
\end{align} and they use $\sum_{t=1}^{T-1} \lrnorm{\x_{t+1}^\ast - \x_t^\ast}$ to bound the regret. Here, $\x_t^\ast = \argmin_{\x\in\Xcal} f_t(\x)$.  It implicitly assumes that the difference between $\x_{t+1}^\ast$ and $\x_t^\ast$ are bounded. It is reasonable for a strongly convex function $f_t$, but may not be guaranteed for a general convex function $f_t$.  Additionally, \cite{Li:2018uy} uses $\lrnorm{\x_{t+1}^\ast - \x_t^\ast}^2$ to bound the switching cost, which is more sensitive to the significant change than $\lrnorm{\x_{t+1}^\ast - \x_t^\ast}$. But, it is  less effective to bound the slight change between them, which is not suitable for many problems such as dynamic control of data centers \cite{Lin:2012:OOS}.  
\end{itemize}

\subsection{Algorithm}

We use mirror descent \cite{BECK2003167} in the online setting, and present the algorithm MD-OA for the OA setting and the algorithm MD-OCO for the OCO setting, respectively.

As illustrated in Algorithms \ref{algo_proximal_oa} and \ref{algo_proximal_oco}, both MD-OA and MD-OCO are performed iteratively. For every round, MD-OA first observes the loss function $f_t$, and then makes the decision $\x_t$ at the $t$-th round.  But, MD-OCO first makes the decision $\x_t$, and then observe the loss function $f_t$. Therefore, MD-OA usually makes the decision based on the observed $f_t$ for the current round, but MD-OCO has to predict a decision for the next round based on the received $f_t$. 

Note that both MD-OA and MD-OCO requires to solve a convex optimizaiton problem to update $x$. The complexity is dominated by the domain $\Xcal$ and the distance function $\Phi$. Besides, both of them lead to $\Ocal{d}$ memory cost. They lead to  comparable cost of computation and memory.

\section{Theoretical analysis}
\label{sect_theoretical_analysis}
In this section, we present our main analysis results about the proposed dynamic regret for both MD-OA and MD-OCO, and discuss the difference between them. 

\subsection{New bounds for dynamic regret with switching cost}

The upper bound of dynamic regret for MD-OA is presented as follows.

\begin{Theorem}
\label{theorem_regret_oa_upper_bound}

Choose $\gamma  = \min \left\{ \frac{\mu}{L}, T^{-\frac{1}{1+\sigma}} D^{\frac{1}{1+\sigma}} \right\}$ in Algorithm \ref{algo_proximal_oa}. Under Assumption \ref{assumption_basic}, we have 
\begin{align}
\nonumber
\sup_{\{f_t\}_{t=1}^T \in \Fcal^T}\Rcal_D^{\textsc{MD-OA}}\lesssim T^{\frac{1}{\sigma+1}}D^{\frac{\sigma}{\sigma+1}} + T^{\frac{1}{\sigma+1}} D^{-\frac{1}{\sigma+1}}.
\end{align} That is,  Algorithm \ref{algo_proximal_oa} yields $\Ocal{T^{\frac{1}{\sigma+1}}D^{\frac{\sigma}{\sigma+1}}}$ dynamic regret with switching cost.

\end{Theorem}

\begin{algorithm}[!t]
   \caption{ MD-OA: Online Mirror Descent for OA.}
   \label{algo_proximal_oa}
   \begin{algorithmic}[1]
   \Require The learning rate $\gamma$, and the number of rounds $T$.
       \For {$t=1,2, ..., T$}
            \State Observe the loss function $f_t$. \Comment{\textit{Observe $f_t$ first.}}
            \State Query a gradient $\hat{\g}_t \in \nabla f_{t}(\x_{t-1})$.
            \State $\x_t = \argmin_{\x \in \Xcal} \lrangle{\hat{\g}_t, \x - \x_{t-1}} + \frac{1}{\gamma}B_{\Phi}(\x, \x_{t-1})$. \Comment{\textit{Play a decision after knowing $f_t$.}}
       \EndFor
       \State \textbf{return} $\x_{T}$
   \end{algorithmic}
\end{algorithm}

\begin{algorithm}[!t]
   \caption{ MD-OCO: Online Mirror Descent for OCO.}
   \label{algo_proximal_oco}
   \begin{algorithmic}[1]
   \Require The learning rate $\eta$, the number of rounds $T$, and $\x_0$.
       \For {$t=0,1, ..., T-1$}
        \State Play $\x_t$. \Comment{\textit{Play a decision first before knowing $f_t$.}}
        \State Receive a loss function $f_t$.
        \State Query a gradient $\bar{\g}_t \in \nabla f_t(\x_t)$.
        \State $\x_{t+1} = \argmin_{\x \in \Xcal} \lrangle{\bar{\g}_t, \x - \x_t} + \frac{1}{\eta}B_{\Phi}(\x, \x_t)$. 
       \EndFor
       \State \textbf{return} $\x_{T}$
   \end{algorithmic}
\end{algorithm}

\begin{Remark}
When $\sigma = 1$, MD-OA yields $\Ocal{\sqrt{TD}}$ dynamic regret, which achieves the state-of-the-art result in \cite{pmlr-v75-chen18b}. When $\sigma = 2$, MD-OA yields $\Ocal{T^{\frac{1}{3}}D^{\frac{2}{3}}}$ dynamic regret, which is a new result as far as we know.
\end{Remark}

However, we find different result for MD-OCO. The switching cost does not have an impact on the dynamic regret. 

\begin{Theorem}
\label{theorem_regret_oco_upper_bound}
Choose $\eta = \min\left\{\frac{\mu}{4}, \sqrt{\frac{D+G}{T}}\right \}$ in Algorithm \ref{algo_proximal_oco}. Under Assumption \ref{assumption_basic}, we have
\begin{align}
\nonumber
\sup_{\{f_{t}\}_{t=1}^T \in \Fcal^T} \Rcal_D^{\textsc{MD-OCO}} \lesssim \sqrt{TD}  + \sqrt{T}.
\end{align} That is, Algorithm \ref{algo_proximal_oco} yields $\Ocal{\sqrt{DT} + \sqrt{T}}$ dynamic regret with switching cost.
\end{Theorem}

\begin{Remark}
MD-OCO still yields $\Ocal{\sqrt{TD}  + \sqrt{T}}$ dynamic regret \cite{Gyorgy:2016} when there is no switching cost. It shows that the switching cost does not have an impact on the dynamic regret.
\end{Remark}

Before presenting the discussion, we show that MD-OCO is the optimum for dynamic regret because the lower bound of the problem matches with the upper bound yielded by MD-OCO. 

\begin{Theorem}
\label{theorem_lower_bound_oco}
Under Assumption \ref{assumption_basic}, the lower bound of the dynamic regret for the OCO problem is  
\begin{align}
\nonumber
\inf_{A\in\Acal} \sup_{\{f_{t}\}_{t=1}^T \in \Fcal^T} \Rcal_D^A  = \Omegacal{\sqrt{TD} + \sqrt{T}}.
\end{align}
\end{Theorem}

\begin{Remark}
When there is no switching cost, the lower bound of dynamic regret for OCO is $\Ocal{\sqrt{TD}+\sqrt{T}}$ \cite{Zhao:2018wx}. Theorem \ref{theorem_lower_bound_oco} achieves it for the case of switching cost. It implies that the switching cost does not let the online decision making in the OCO setting become more difficult. 
\end{Remark}

\subsection{Insights}

\textbf{Switching cost has a significant impact on the dynamic regret for the setting of OA.} According to Theorem \ref{theorem_regret_oa_upper_bound}, the switching cost has a significant impact on the dynamic regret of MD-OA. Given a constant $D$, a small $\sigma$ leads to a strong dependence on $T$, and  a large $\sigma$ leads to a weak dependence on $T$. The reason is that  a large $\sigma$ leads to a large learning rate, which is more effective to follow the dynamics in the environment than a small learning rate. 

\textbf{Switching cost does not have an impact on the dynamic regret for the setting of OCO.} According to Theorem 
\ref{theorem_regret_oco_upper_bound} and Theorem \ref{theorem_lower_bound_oco},  the dynamic regret yielded by MD-OCO is tight, and MD-OCO is the optimum for the problem. Although the switching cost exists, the dynamic regret yielded by MD-OCO does not have any difference. 

As we can see, there is a significant difference between the OA setting and the OCO setting. The reasons are presented as follows.
\begin{itemize}
\item MD-OA makes decisions after observing the loss function. It has known the potential operating cost and switching cost for any decision. Thus, it can make decisions to achieve a good tradeoff between the operating cost and switching cost.  
\item MD-OCO make decisions before observing the loss function. It only knows the  historical information and the potential switching cost, and does not know the potential operating cost for any decision at the current round. In the worst case, if the environment provides an adversary loss function to maximize the operating cost based on the decision played by MD-OCO, MD-OCO has to lead to $\Ocal{\sqrt{TD}+\sqrt{T}}$ regret even for the case of no switching cost \cite{Gyorgy:2016}. Although the potential switching cost is known, MD-OCO cannot make a better decision to reduce the regret due to unknown operating cost. 
\end{itemize}

\section{Empirical studies}
\label{sect_empirical_study}
In this section, we evaluate the total regret and the regret caused by switching cost for settings of both OA and OCO by running online mirror decent. Our experiments show the importance of knowing loss function before making a decision. 

\subsection{Experimental settings}
We conduct binary classification by using the logistic regression model. Given an instance $\a\in\RR^d$ and its label $y\in\{1,-1\}$, the loss function is
\begin{align}
\nonumber
f(\x) = \log\lrincir{1+\exp\lrincir{-y\a\Tr\x}}.
\end{align} In experiments, we let $\Phi(\x) = \frac{1}{2}\lrnorm{\x}^2$. 

We test four methods, including MD-OA, i.e., Algorithm \ref{algo_proximal_oa}, and MD-OCO, i.e., Algorithm \ref{algo_proximal_oco}, online balanced descent \cite{pmlr-v75-chen18b} denoted by BD-OA in the experiment, and multiple online gradient descent \cite{Zhang:2017:IDR} denoted by MGD-OCO in the experiment. Both MD-OA and BD-OA are two variants of online algorithm, and similarily both MD-OCO and MGD-OCO are two variants of online convex optimization.  We test those methods on three real datasets: \textit{usenet1}\footnote{\url{http://lpis.csd.auth.gr/mlkd/usenet1.rar}}, \textit{usenet2}\footnote{\url{http://lpis.csd.auth.gr/mlkd/usenet2.rar}}, and \textit{spam}\footnote{\url{http://lpis.csd.auth.gr/mlkd/concept_drift/spam_data.rar}}. The distributions of data streams change over time for those datasets, which is just the dynamic environment as we have discussed. More details about those datasets  and its dynamics are presented at: \url{http://mlkd.csd.auth.gr/concept_drift.html}.

We use the \textit{average loss} to test the regret, because they have the same optimal reference points $\{\y_t^\ast\}_{l=1}^t$. For the $t$-th round, the average loss is defined by 
\begin{align}
\nonumber
\underbrace{\frac{1}{t}\sum_{l=1}^t \log\lrincir{1+\exp\lrincir{-\y_l\A_l\Tr\x_l}} }_{\text{average loss caused by operating cost}} + \underbrace{\frac{1}{t}\sum_{l=0}^{t-1}\lrnorm{\x_{l+1} - \x_l}}_{\text{average loss caused by switching cost}},
\end{align} where $\A_l$ is the instance at the $l$-th round, and $\y_l$ is its label. Besides, we evaluate the average loss caused by operating cost separately, and denote it by OL. Similarly, SL represents the average loss caused by switching cost.

In experiment, we set $D = 10$. Since $G$, $\mu$, and $L$ are usually not known in practical scenarios, the learning rate is set by the following heuristic rules. We choose the learning rate $\gamma_t = \eta_t = \frac{\delta}{\sqrt{t}}$ for the $t$-th iteration, where $\delta$ is a given constants by the following rules. First, we set a large value $\delta = 10$. Then, we iteratively adjust the value of $\delta$ by $\delta \leftarrow \delta/2$ when $\delta$ cannot let the average loss converge.  If the first appropriate $\delta$ can let the average loss converge, it is finally chosen as the optimal learning rate.  We use the similar heuristic method to determine other parameters, e.g., the number of inner iterations in MGD-OCO. Finally, the mirror map function is $\frac{1}{2}\lrnorm{\cdot}^2$ for BD-OA.
 
\begin{figure*}[!]
\setlength{\abovecaptionskip}{0pt}
\setlength{\belowcaptionskip}{0pt}
\centering 
\subfigure[\textit{usenet1}, total loss, $\sigma = 1$]{\includegraphics[width=0.32\columnwidth]{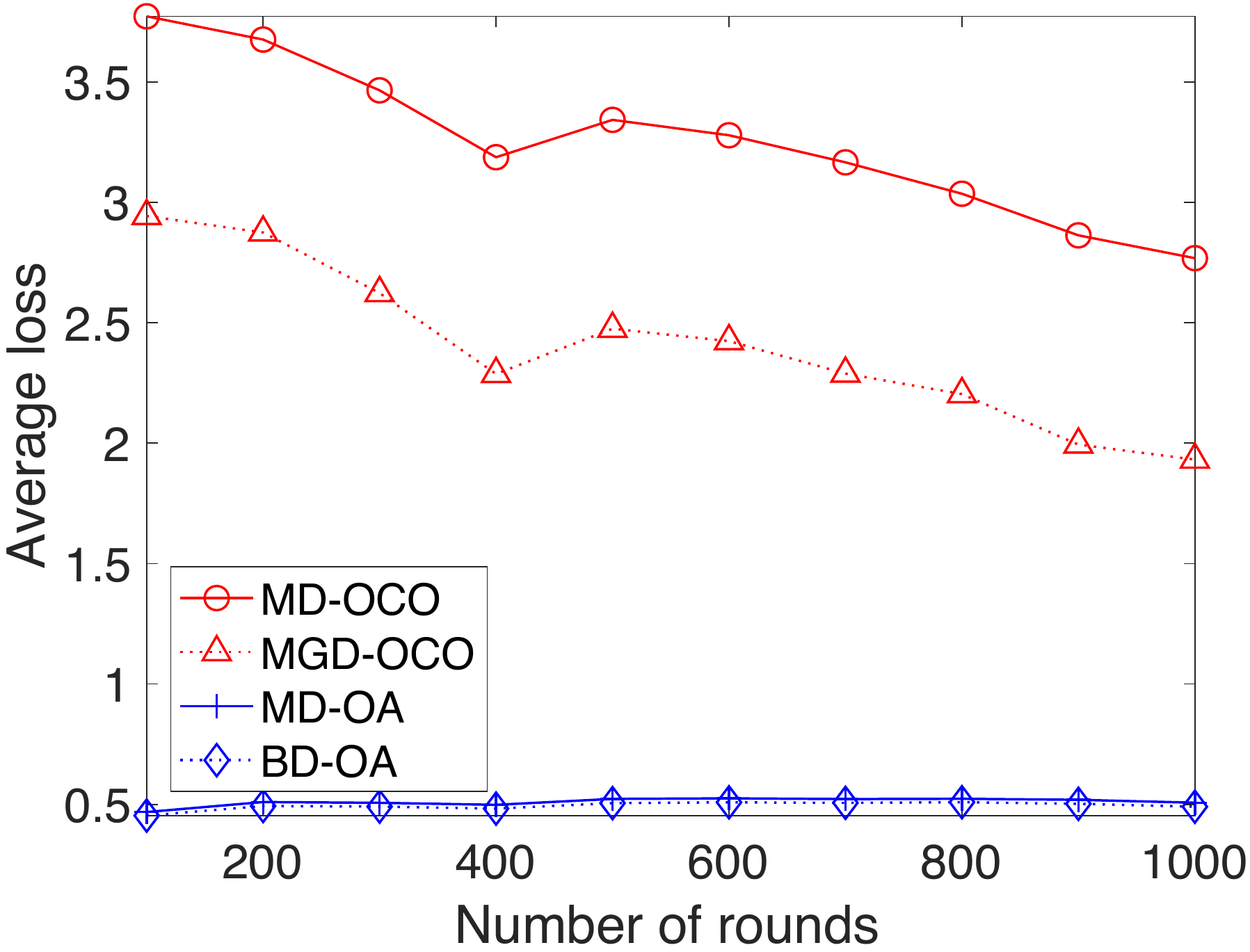}\label{figure_ave_loss_usenet1_sigma1}}
\subfigure[\textit{usenet1}, total loss, $\sigma = 1.5$]{\includegraphics[width=0.32\columnwidth]{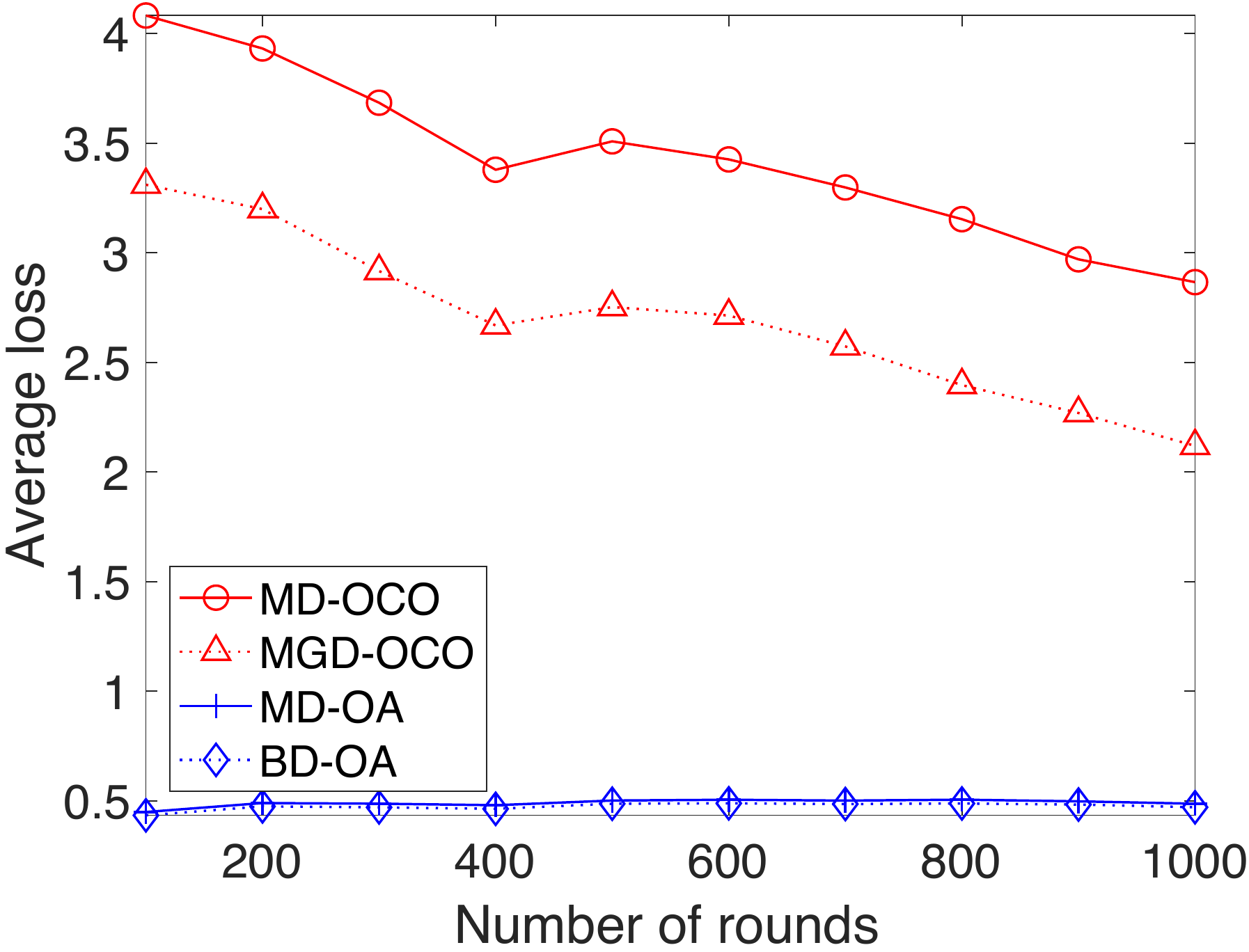}\label{figure_ave_loss_usenet1_sigma1_5}}
\subfigure[\textit{usenet1}, total loss, $\sigma = 2$]{\includegraphics[width=0.32\columnwidth]{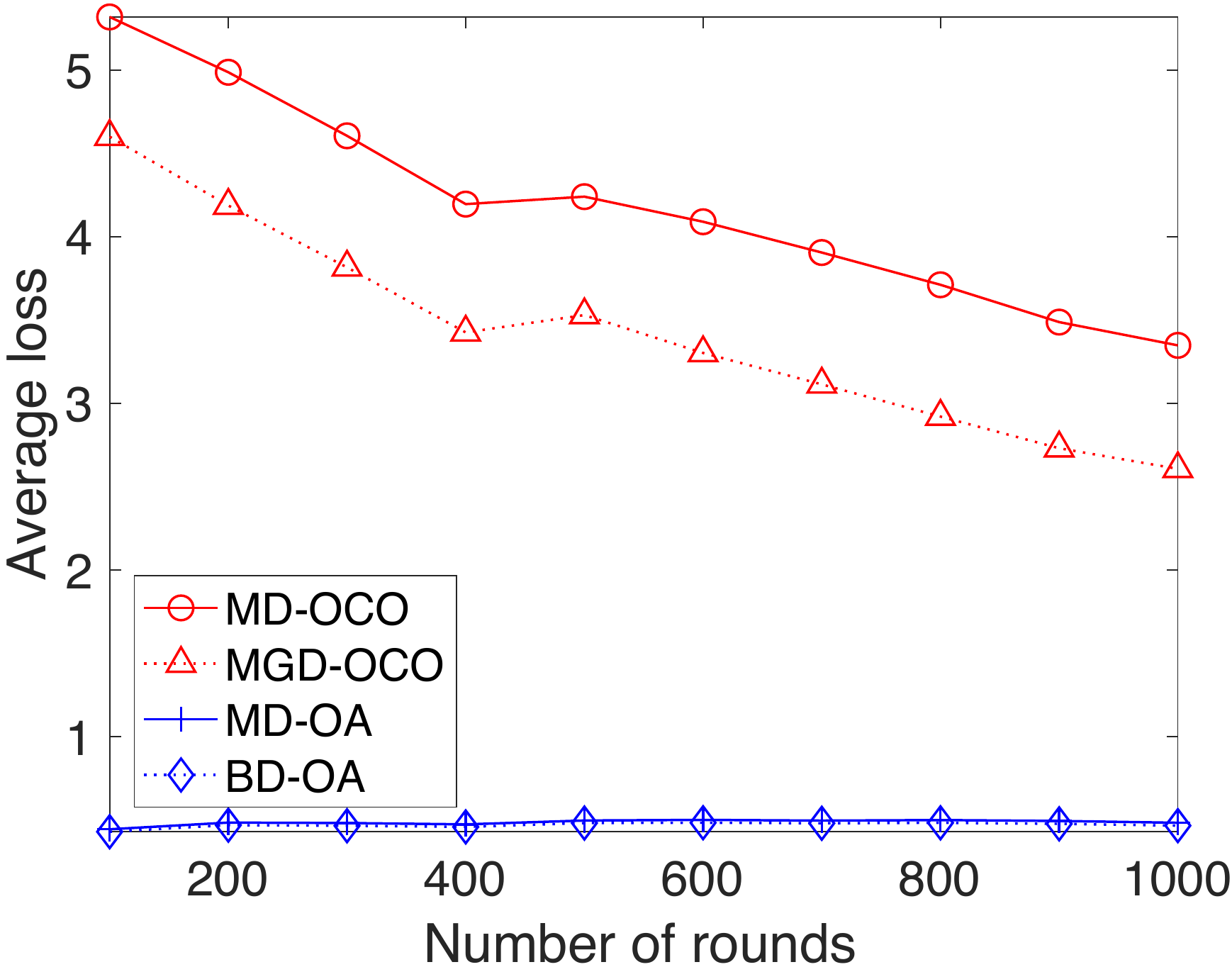}\label{figure_ave_loss_usenet1_sigma2}}
\subfigure[\textit{usenet2}, total loss, $\sigma = 1$]{\includegraphics[width=0.32\columnwidth]{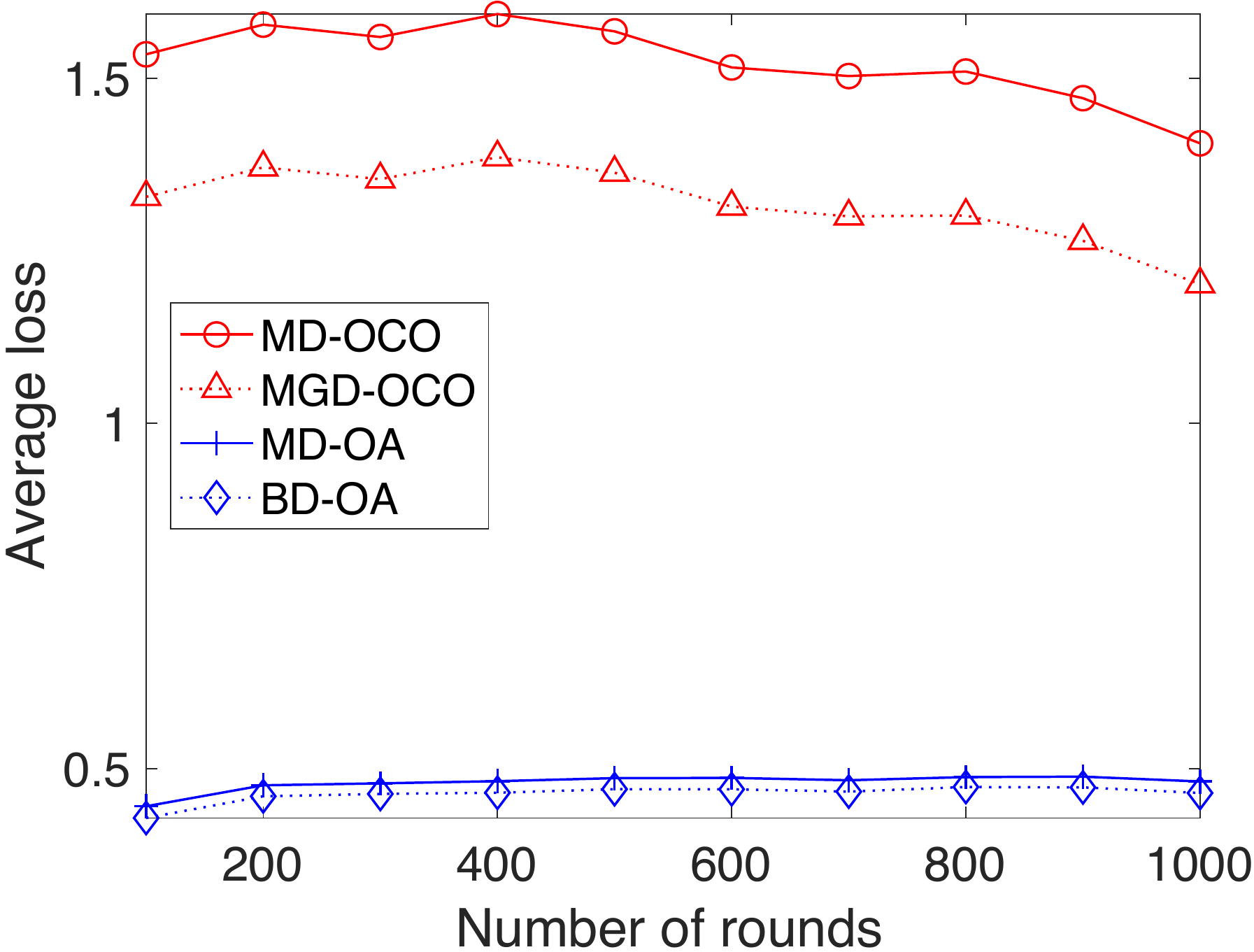}\label{figure_ave_loss_usenet2_sigma1}}
\subfigure[\textit{usenet2}, total loss, $\sigma = 1.5$]{\includegraphics[width=0.32\columnwidth]{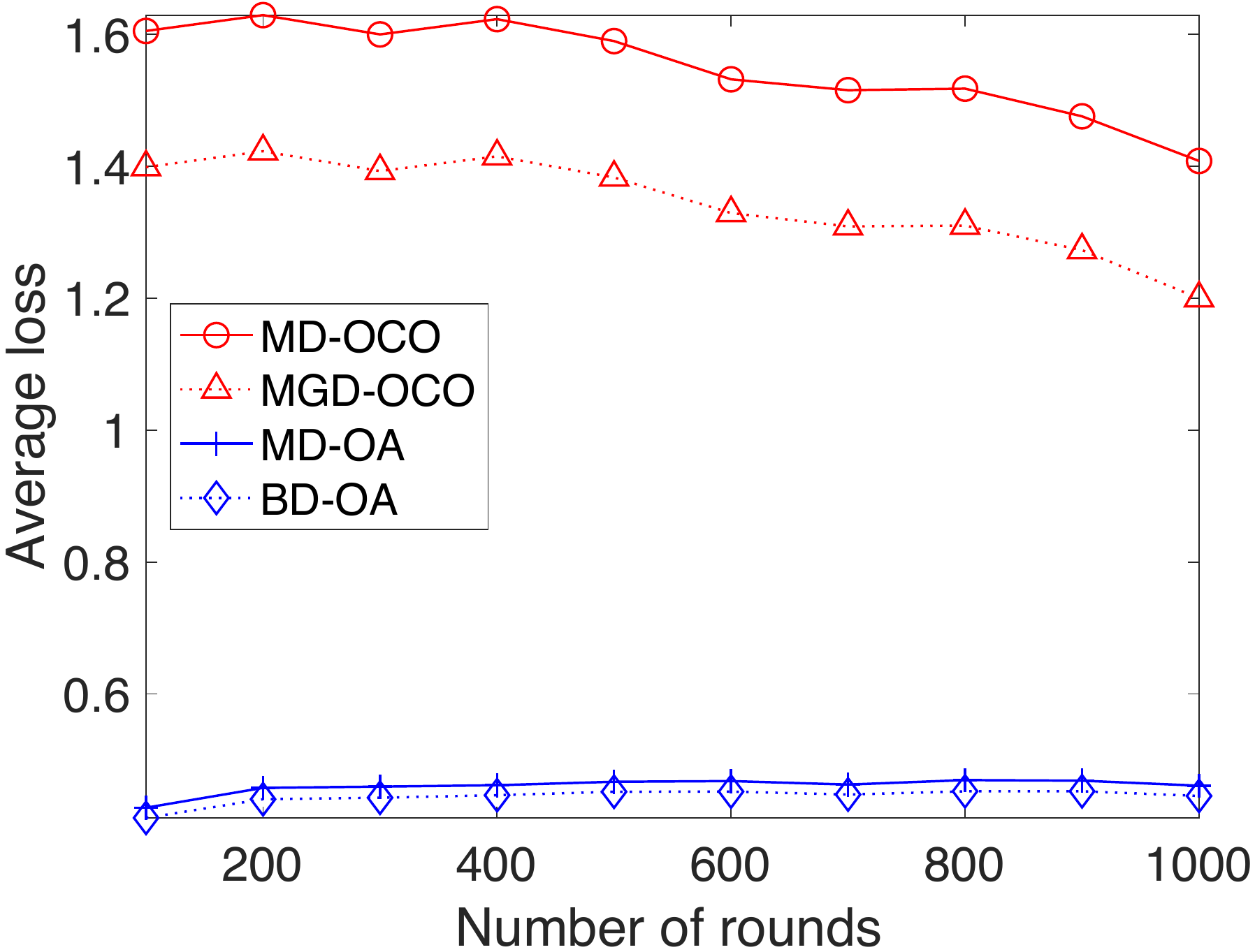}\label{figure_ave_loss_usenet2_sigma1_5}}
\subfigure[\textit{usenet2}, total loss, $\sigma = 2$]{\includegraphics[width=0.32\columnwidth]{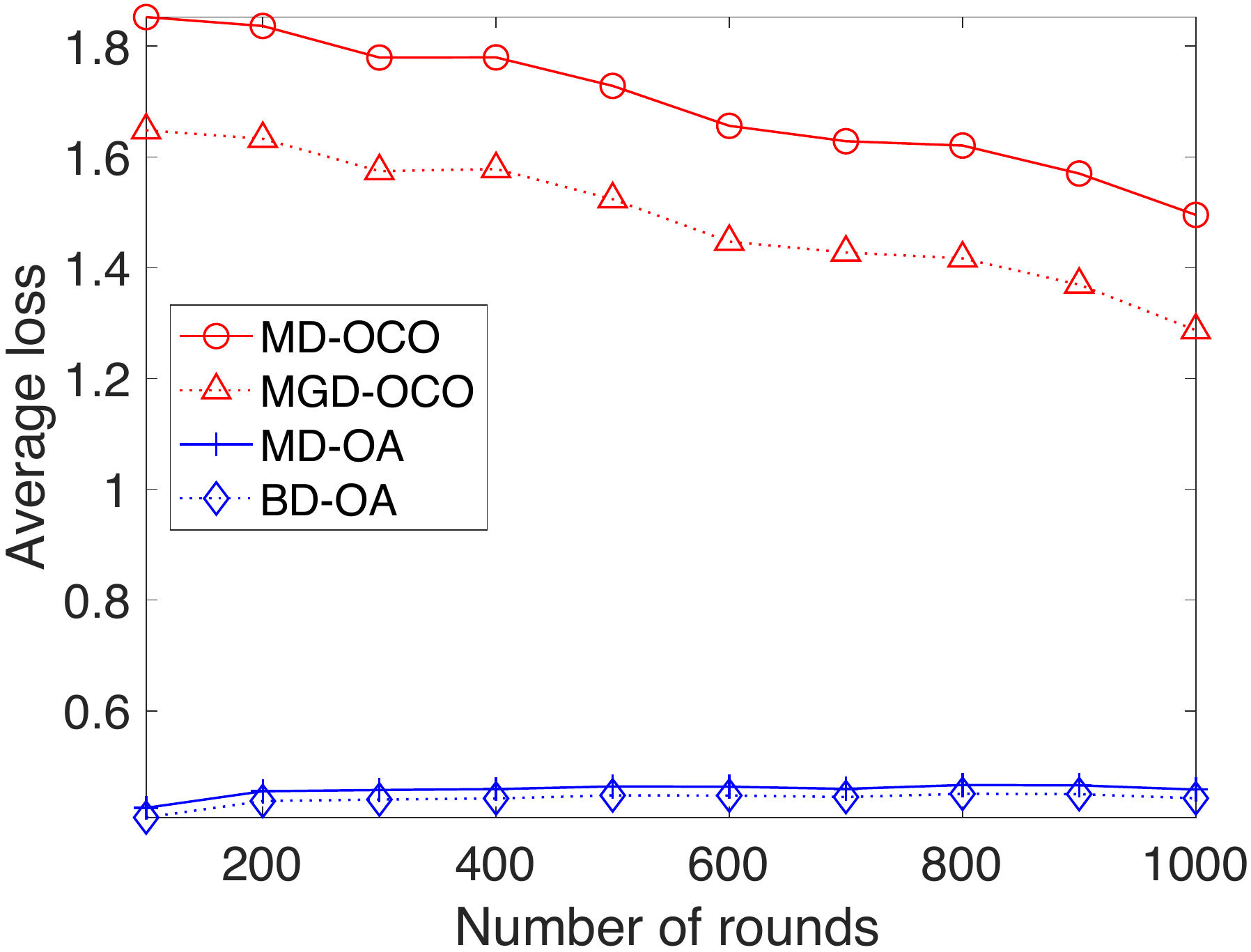}\label{figure_ave_loss_usenet2_sigma2}}
\subfigure[\textit{spam}, total loss, $\sigma = 1$]{\includegraphics[width=0.32\columnwidth]{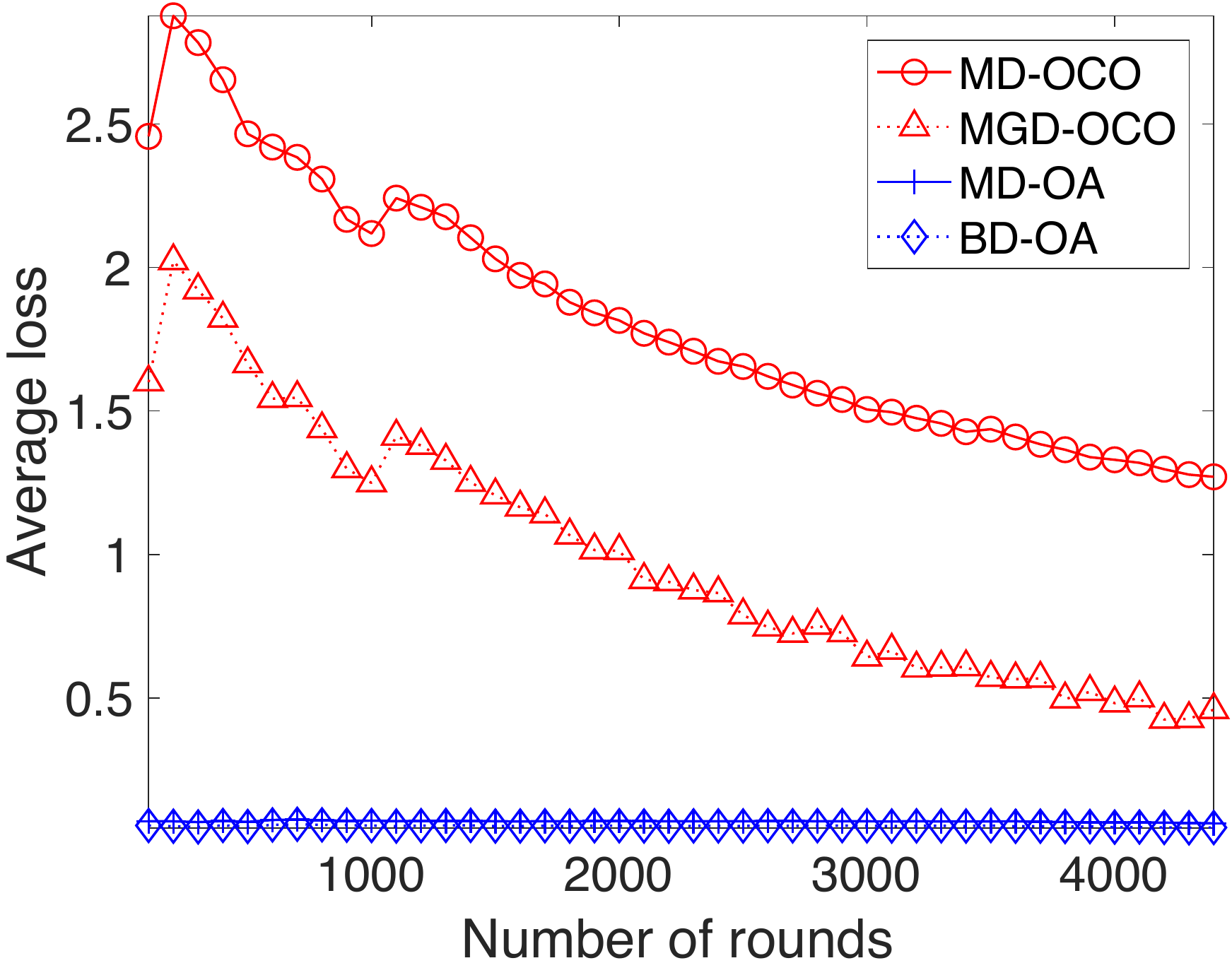}\label{figure_ave_loss_spam_sigma1}}
\subfigure[\textit{spam}, total loss, $\sigma = 1.5$]{\includegraphics[width=0.32\columnwidth]{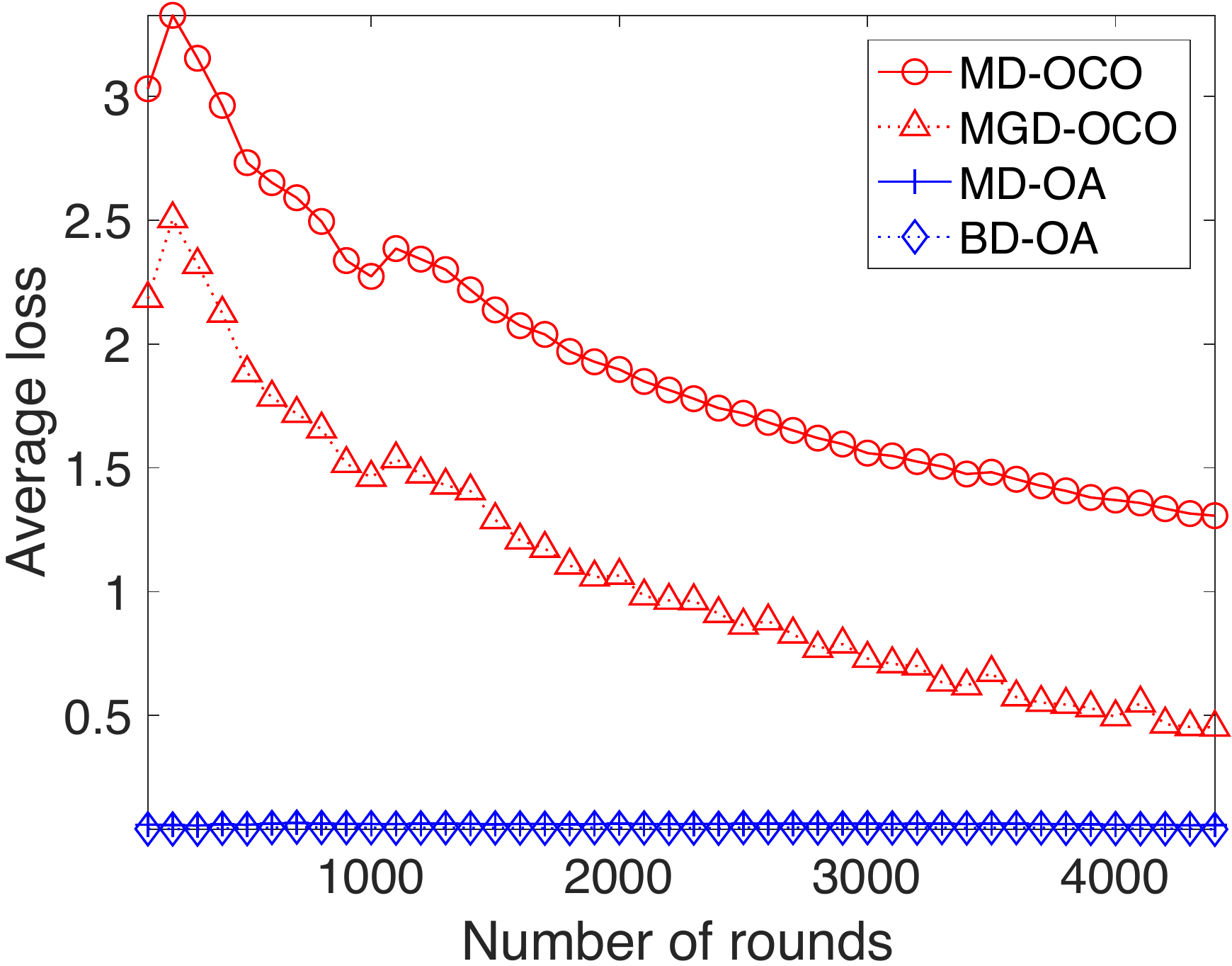}\label{figure_ave_loss_spam_sigma1_5}}
\subfigure[\textit{spam}, total loss, $\sigma = 2$]{\includegraphics[width=0.32\columnwidth]{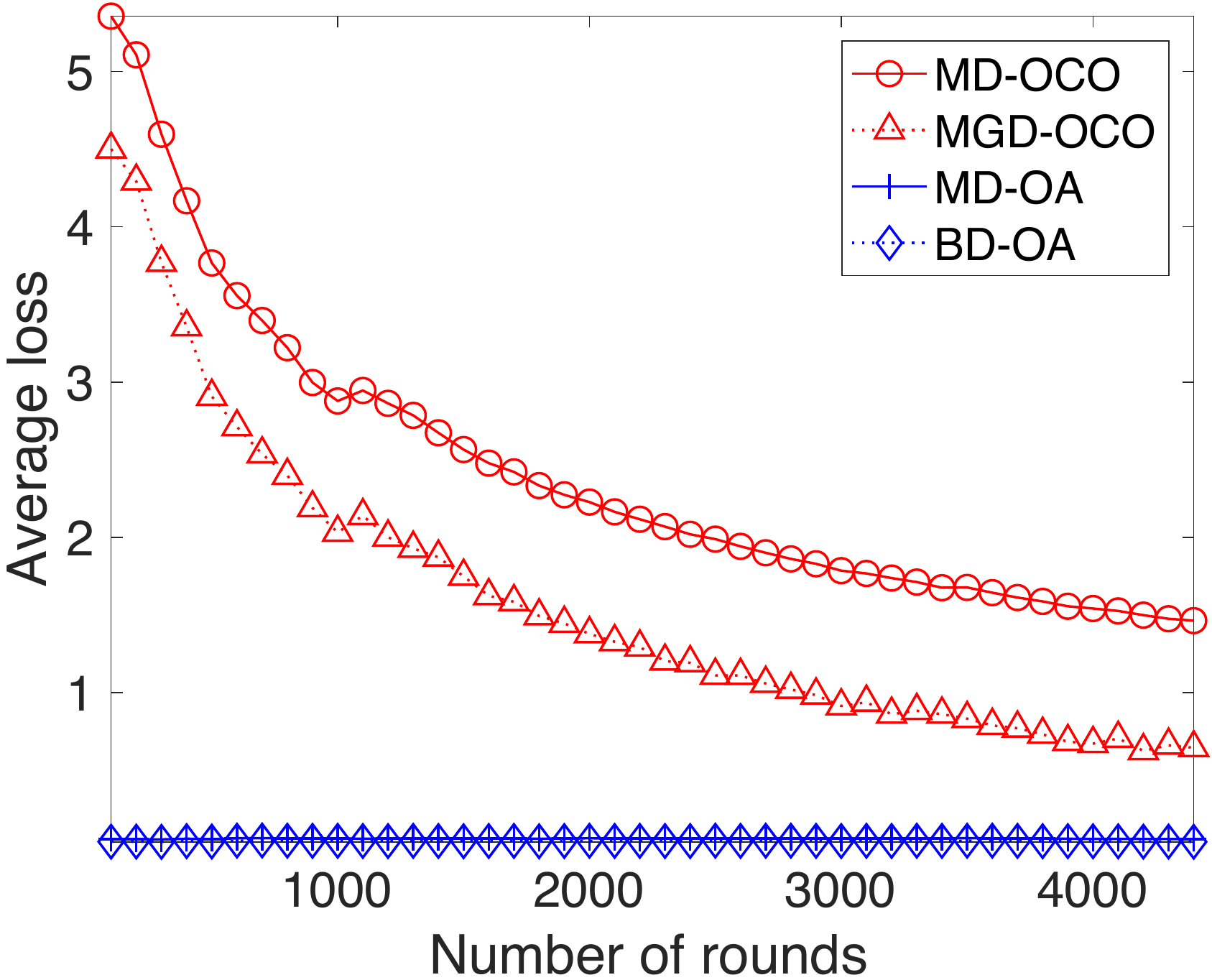}\label{figure_ave_loss_spam_sigma2}}
\caption{ OCO methods leads to large average loss than OA methods. }
\label{figure_ave_loss}
\end{figure*}

\begin{figure*}[!]
\setlength{\abovecaptionskip}{0pt}
\setlength{\belowcaptionskip}{0pt}
\centering 
\subfigure[\textit{usenet1}, separated loss, $\sigma = 1$]{\includegraphics[width=0.32\columnwidth]{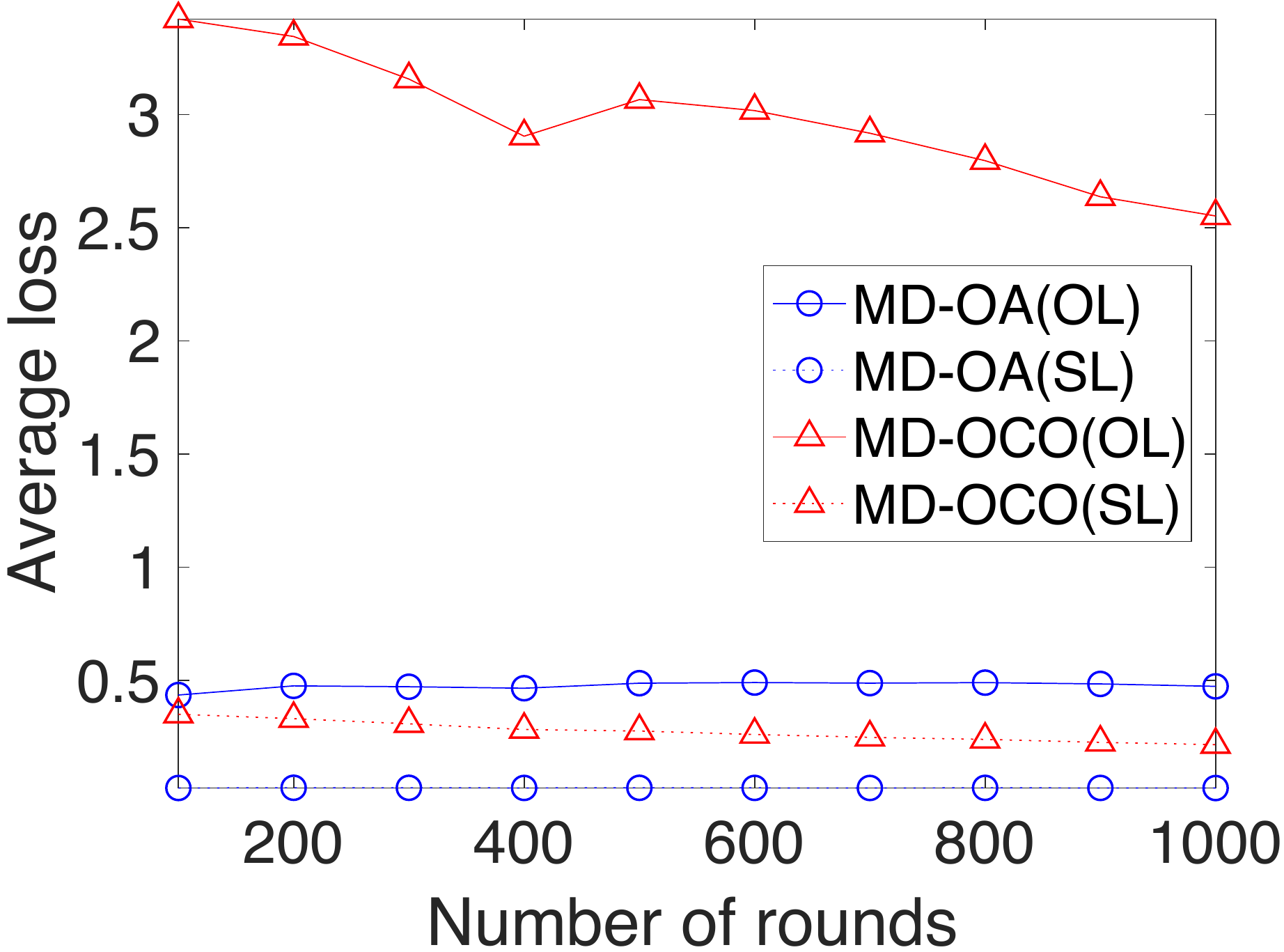}\label{figure_ave_loss_usenet1_sigma1_separate}}
\subfigure[\textit{usenet1}, separated loss, $\sigma = 1.5$]{\includegraphics[width=0.32\columnwidth]{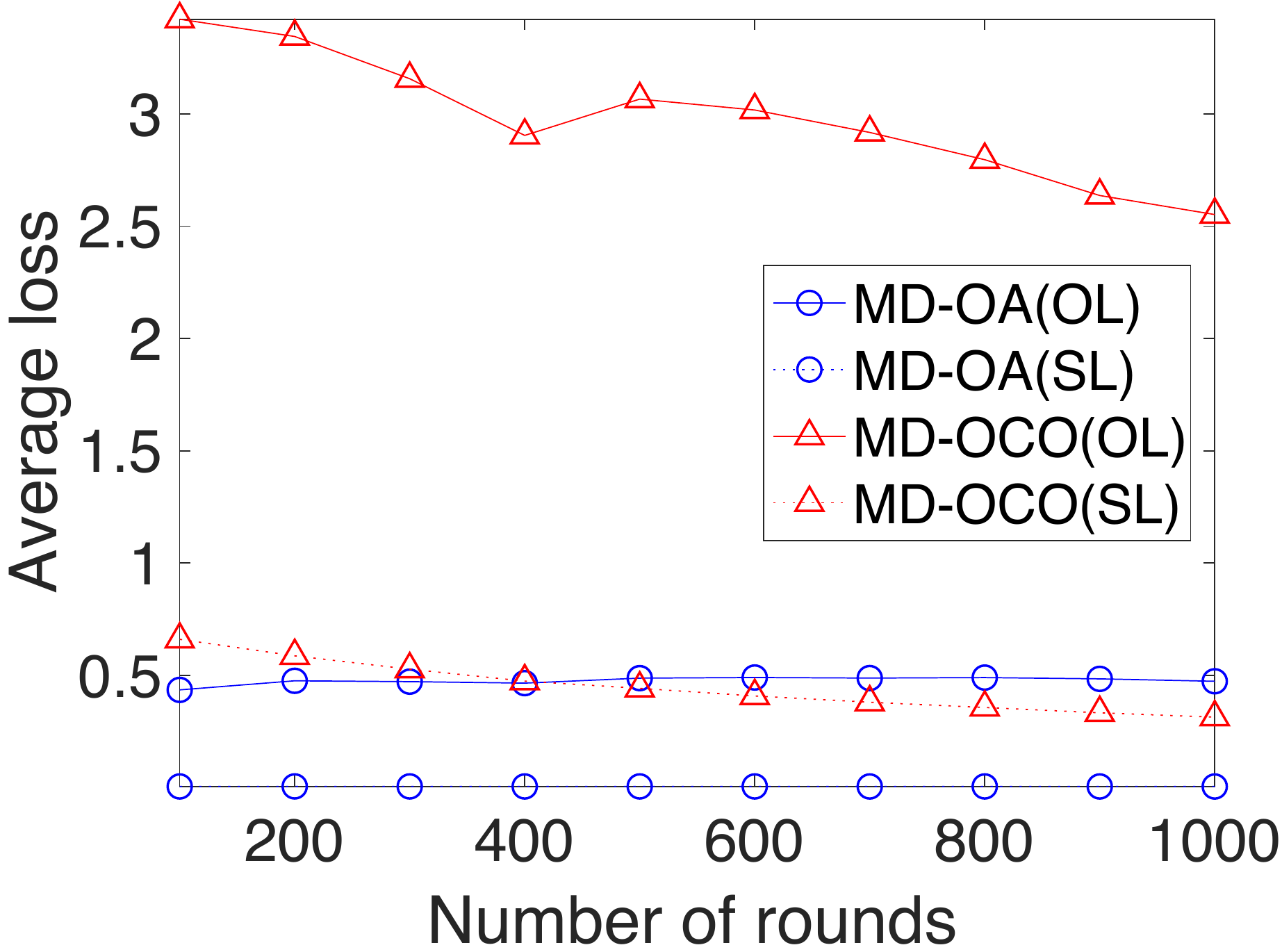}\label{figure_ave_loss_usenet1_sigma1_5_separate}}
\subfigure[\textit{usenet1}, separated loss, $\sigma = 2$]{\includegraphics[width=0.32\columnwidth]{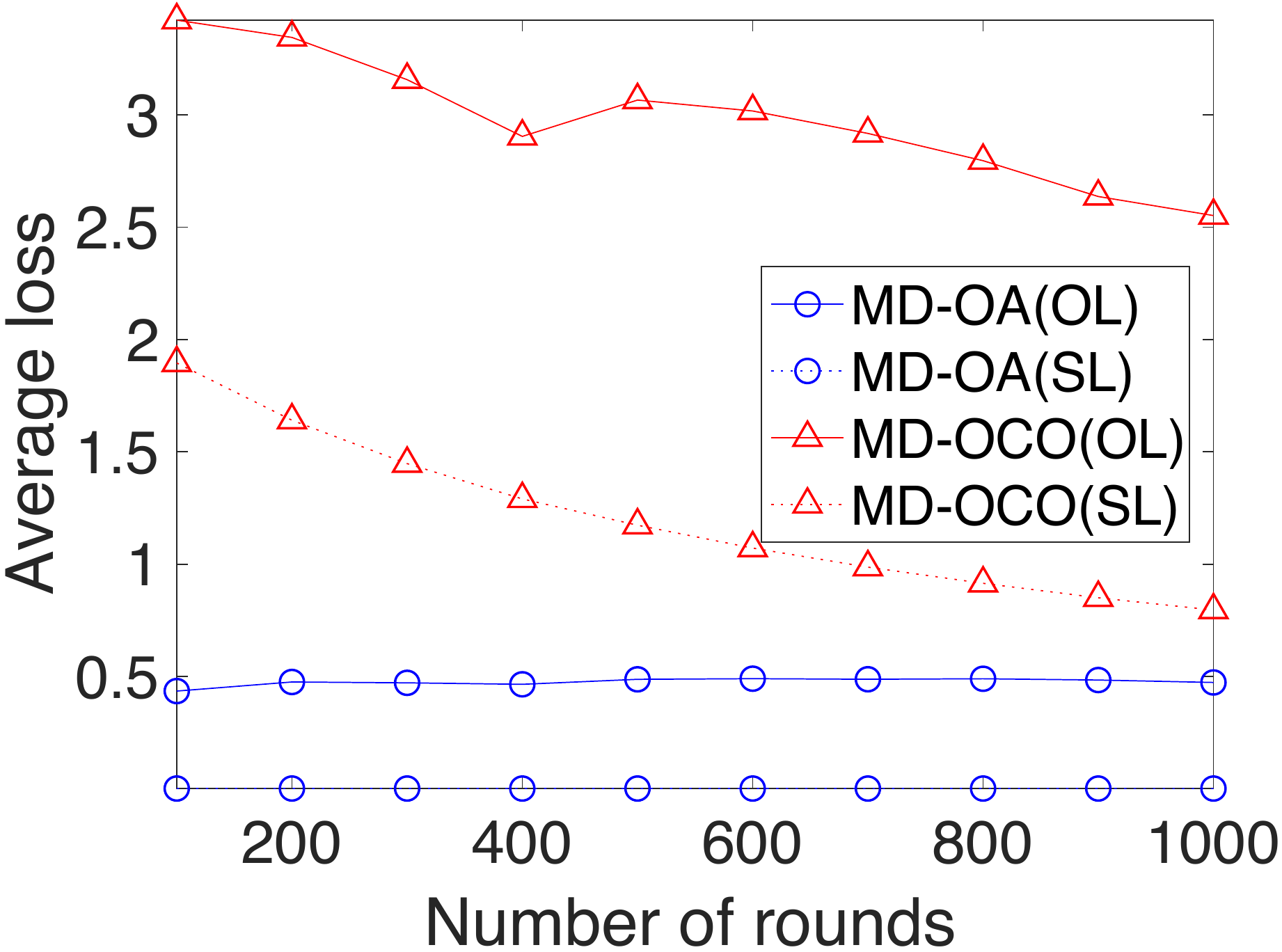}\label{figure_ave_loss_usenet1_sigma2_separate}}
\subfigure[\textit{usenet2}, separated loss, $\sigma = 1$]{\includegraphics[width=0.32\columnwidth]{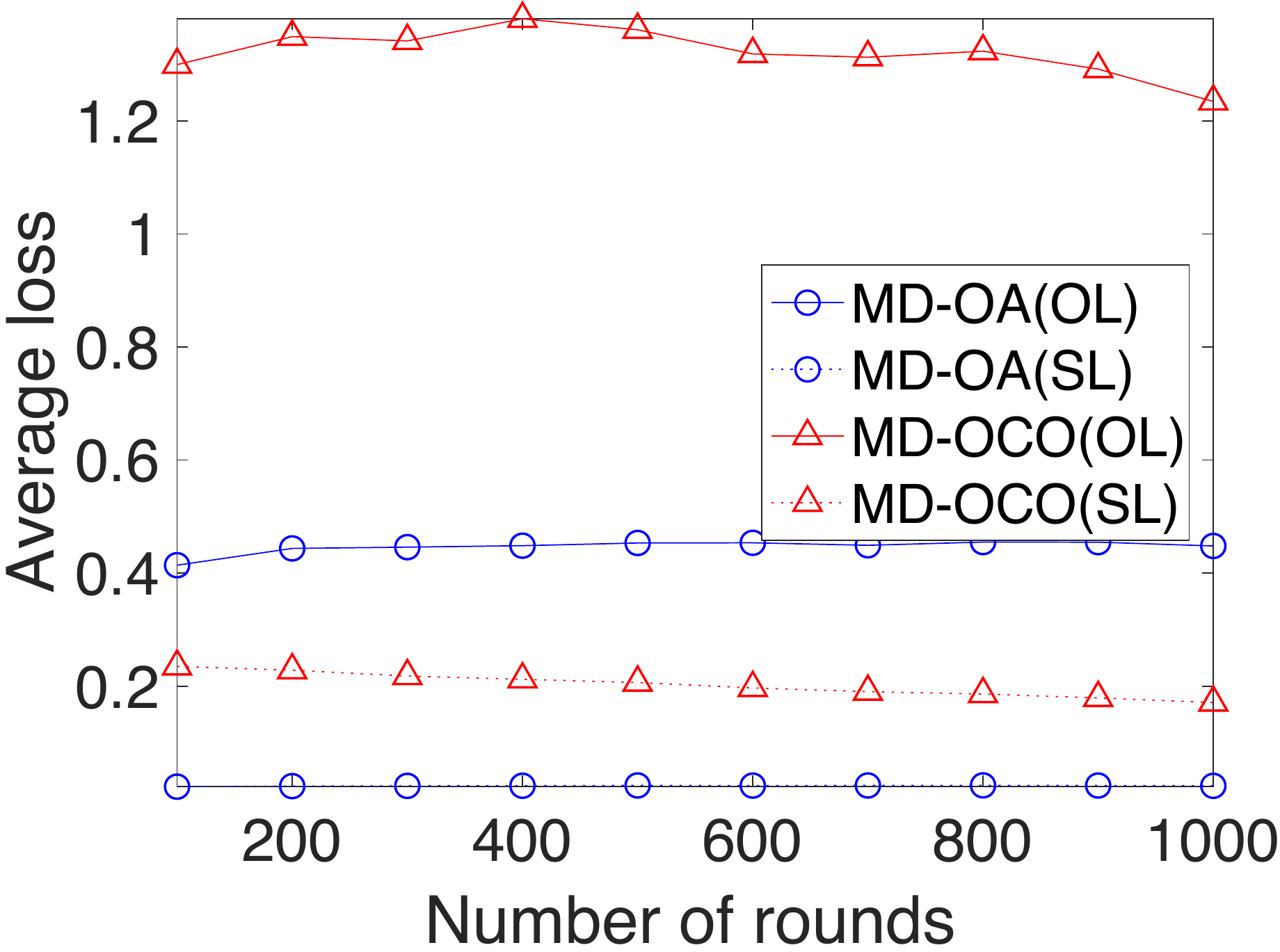}\label{figure_ave_loss_usenet2_sigma1_separate}}
\subfigure[\textit{usenet2}, separated loss, $\sigma = 1.5$]{\includegraphics[width=0.32\columnwidth]{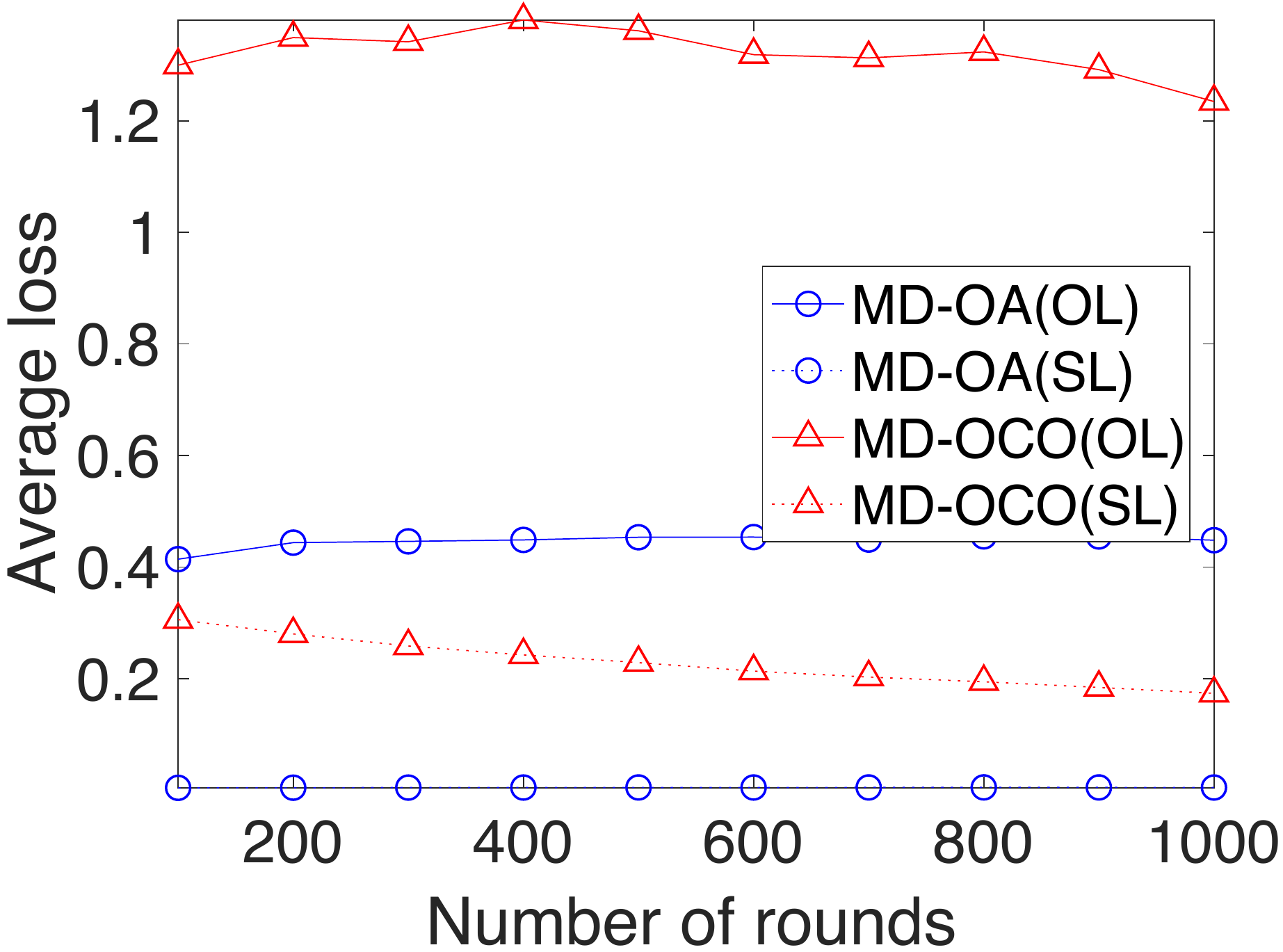}\label{figure_ave_loss_usenet2_sigma1_5_separate}}
\subfigure[\textit{usenet2}, separated loss, $\sigma = 2$]{\includegraphics[width=0.32\columnwidth]{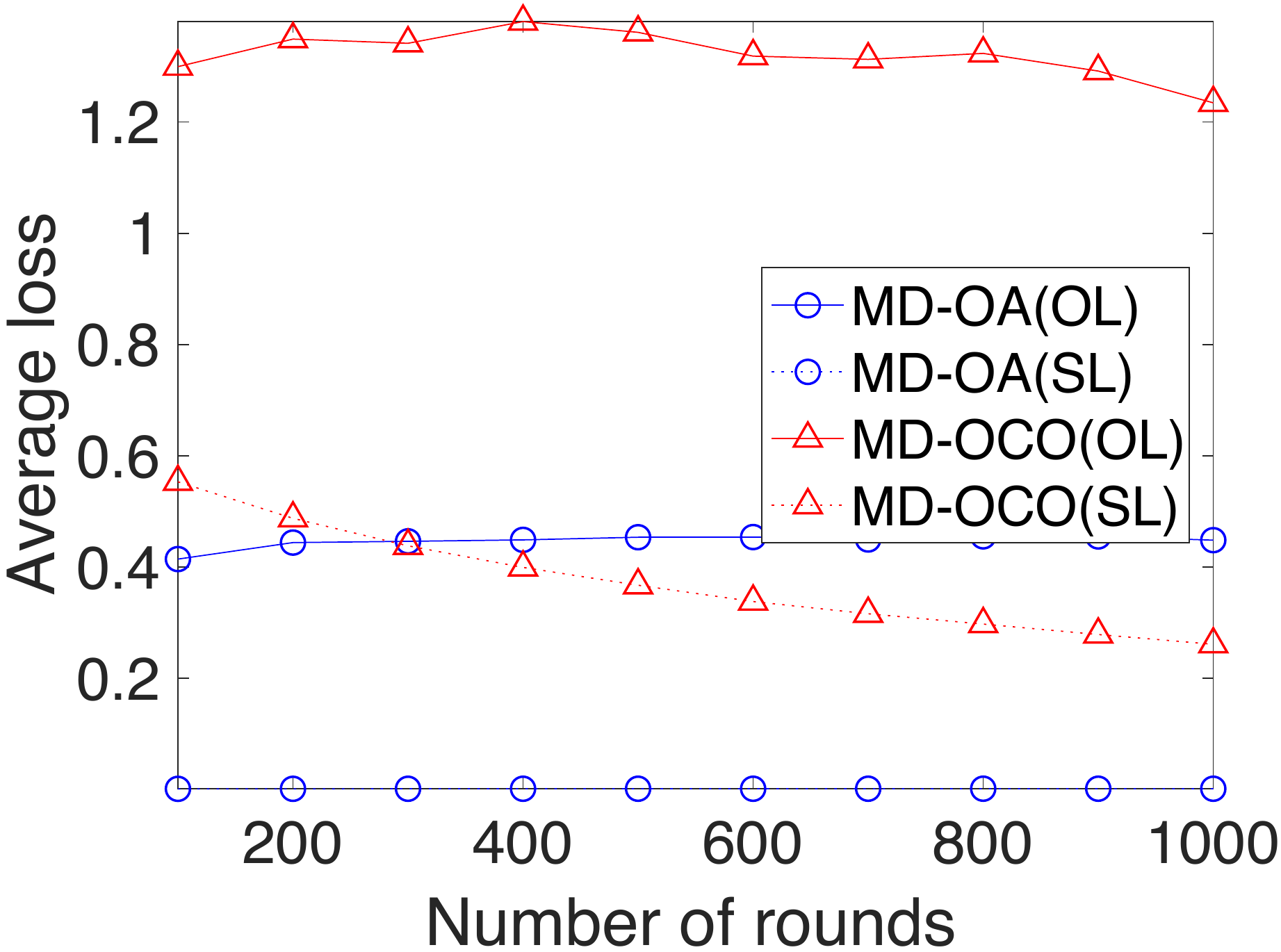}\label{figure_ave_loss_usenet2_sigma2_separate}}
\subfigure[\textit{spam}, separated loss, $\sigma = 1$]{\includegraphics[width=0.32\columnwidth]{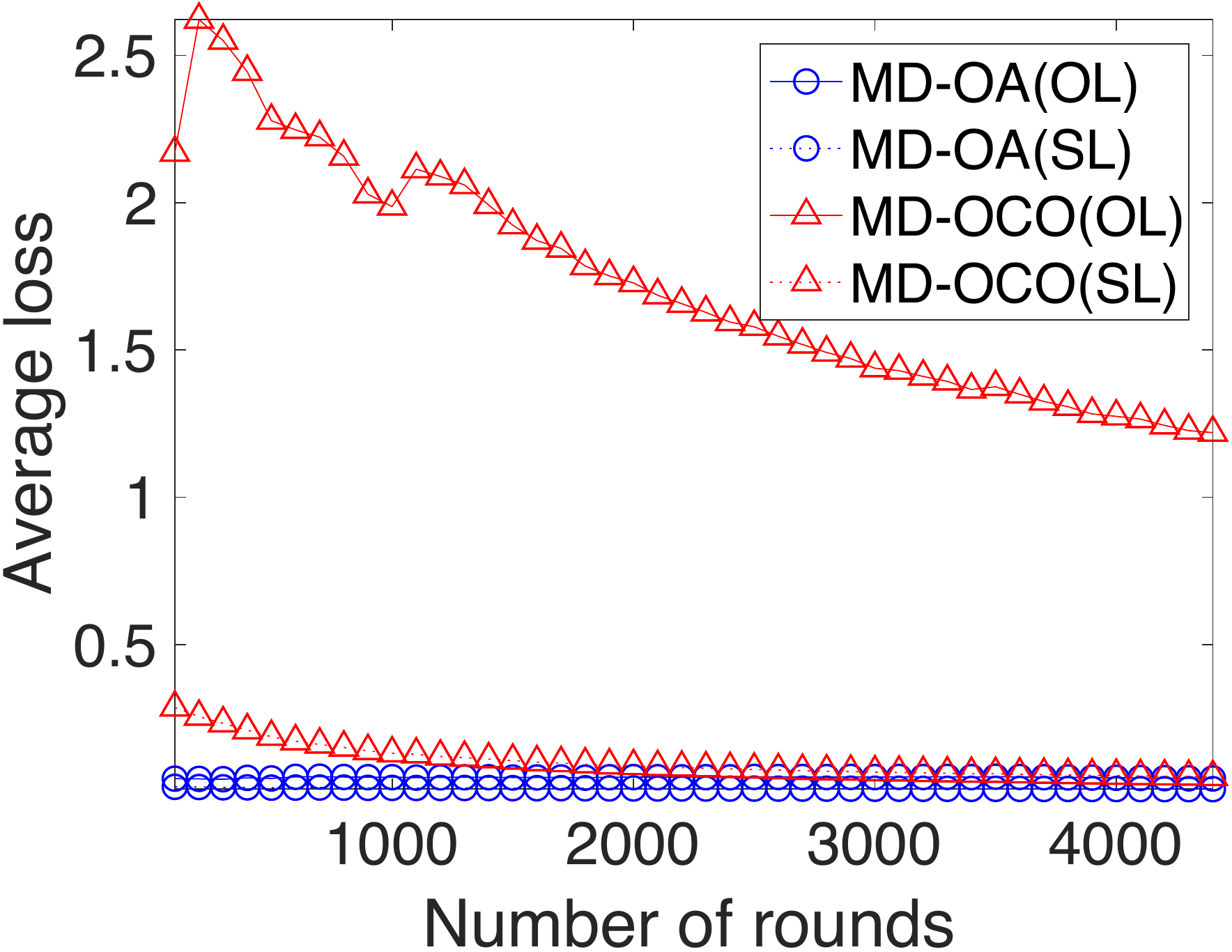}\label{figure_ave_loss_spam_sigma1_separate}}
\subfigure[\textit{spam}, separated loss, $\sigma = 1.5$]{\includegraphics[width=0.32\columnwidth]{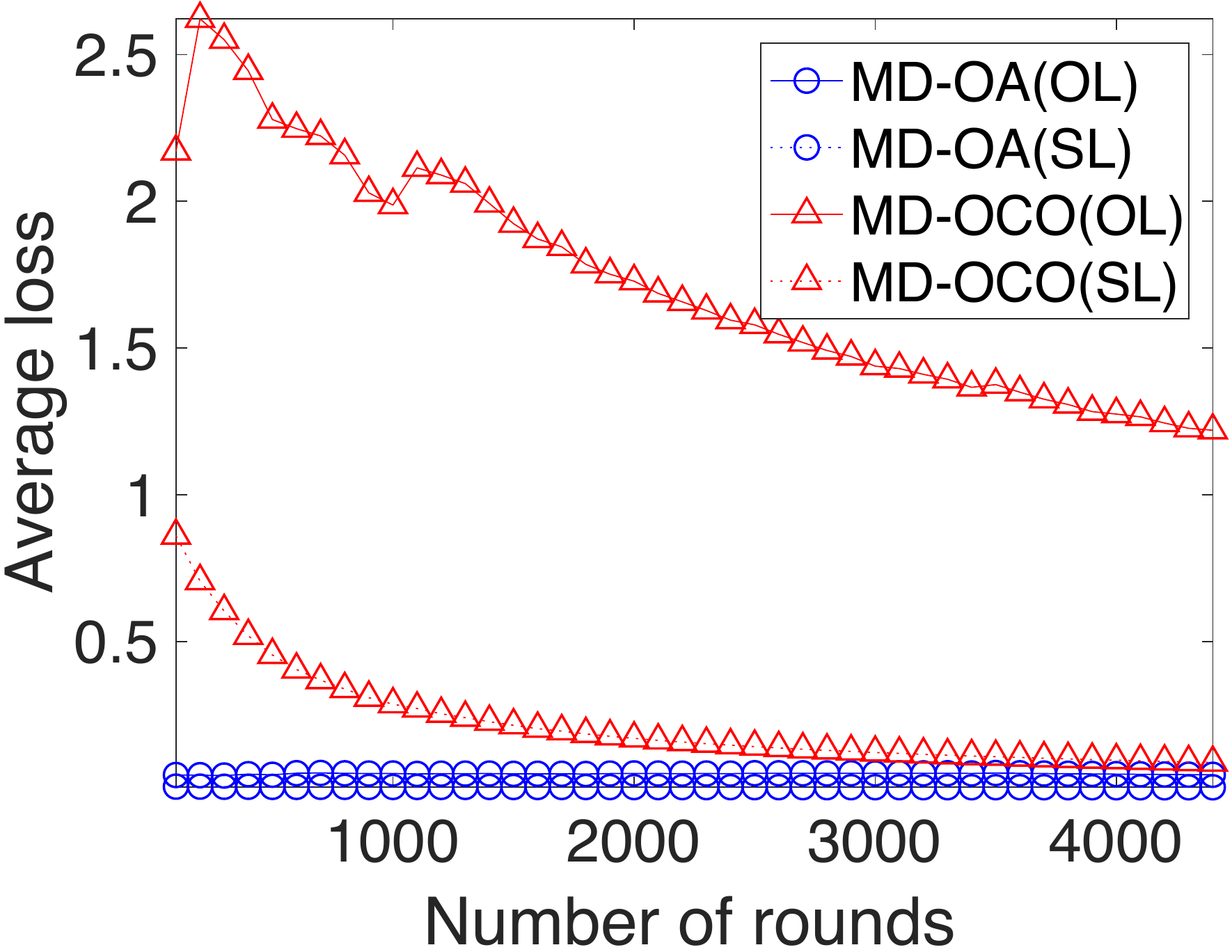}\label{figure_ave_loss_spam_sigma1_5_separate}}
\subfigure[\textit{spam}, separated loss, $\sigma = 2$]{\includegraphics[width=0.32\columnwidth]{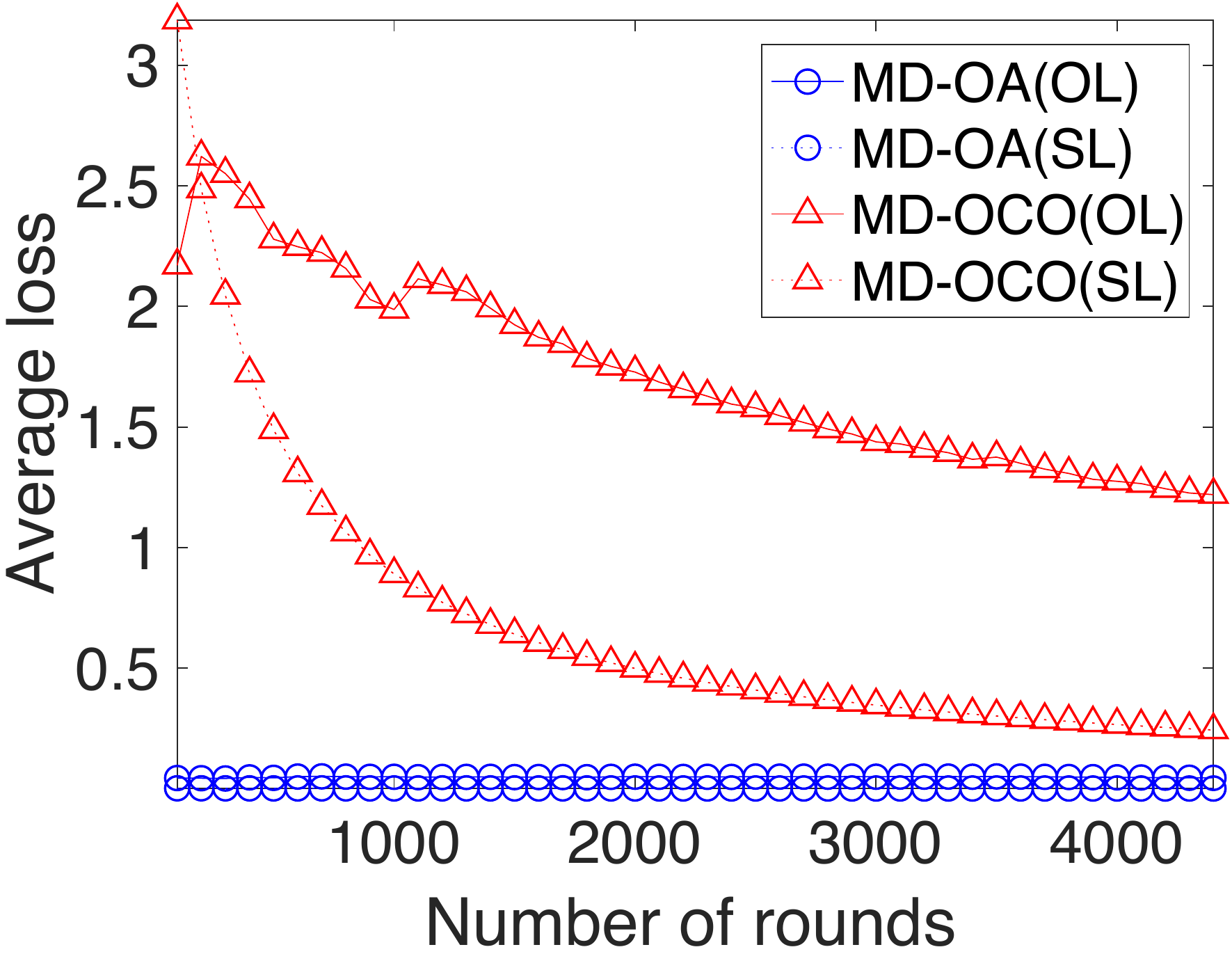}\label{figure_ave_loss_spam_sigma2_separate}}
\caption{ Comparing with MD-OCO, The superiority of MD-OA becomes significant for a large $\sigma$. }
\label{figure_ave_loss_separate}
\end{figure*}

\begin{figure}[!]
\setlength{\abovecaptionskip}{0pt}
\setlength{\belowcaptionskip}{0pt}
\centering 
\includegraphics[width=0.55\columnwidth]{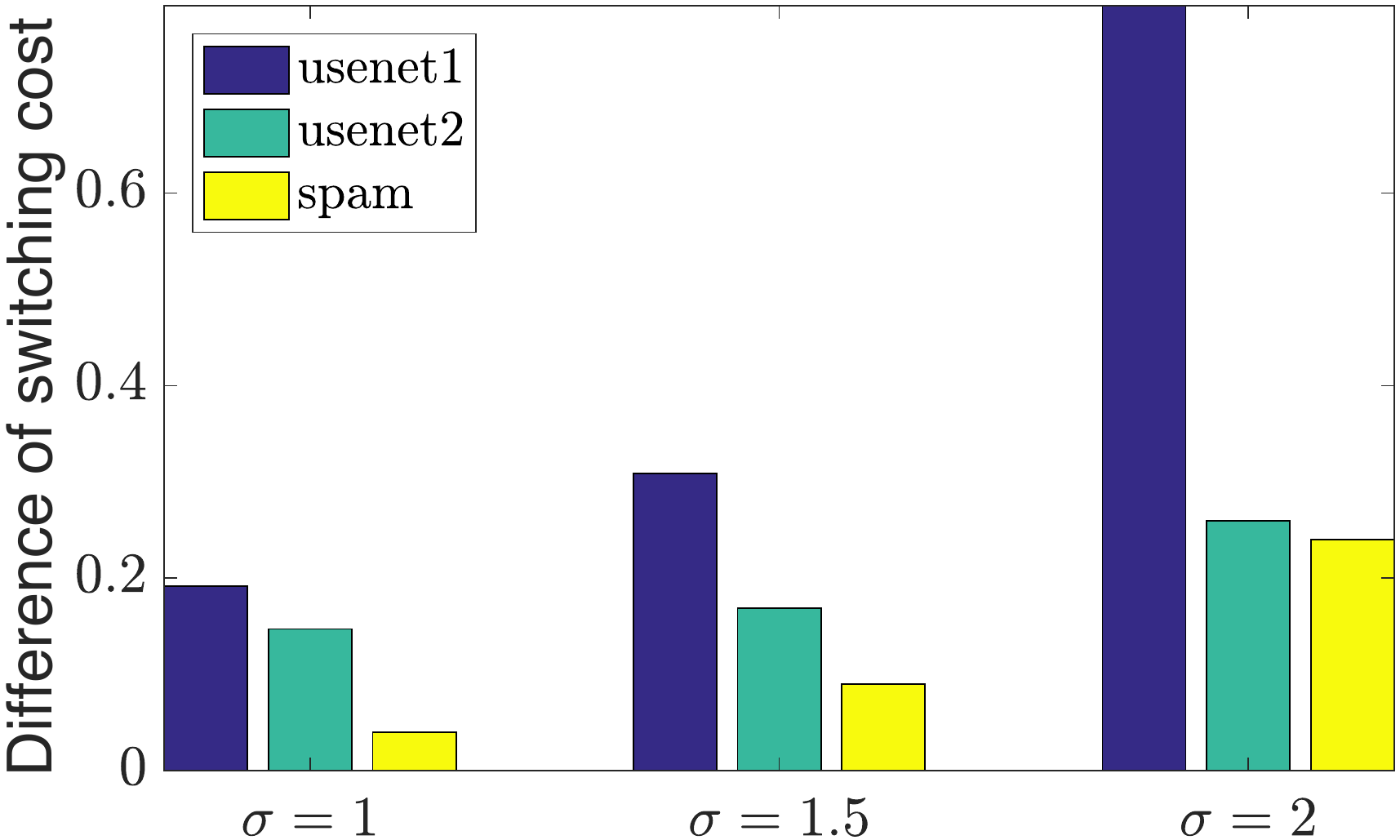}
\caption{ MD-OCO leads to more average loss caused by switching cost than MD-OA, especially for a large $\sigma$. }
\label{figure_ave_loss_difference_switching_cost}
\end{figure}

\subsection{Numerical results}

As shown in Figure \ref{figure_ave_loss}, both MD-OA and BD-OA are much more effcetive than MD-OCO and MGD-OCO to decrease the average loss during a few rounds of begining. Those OA methods yield much smaller average loss than OCO methods. The reason is that OA knows the loss function $f_t$ before making decision $\x_t$. But, OCO has to make decision before know the loss function. Benefiting from knowing the loss function $f_t$, OA reduces the average loss more efffectively than OCO. It matches with our theoretical analysis. That is, Algorithm \ref{algo_proximal_oa} leads to $\Ocal{T^{\frac{1}{1+\sigma}}D^{\frac{\sigma}{1+\sigma}}}$ regret, but Algorithm \ref{algo_proximal_oco} leads to $\Ocal{\sqrt{TD}+\sqrt{T}}$ regret. When $\sigma \ge 1$, OA tends to lead to smaller regret than OCO. The reason is that OA knows the potential loss before playing a decision for every round. But, OCO works in an adversary environment, and it has to play a decision before knowing the potential loss.  Thus, OA is able to play a better decision than OCO to decrease the loss. Additionally, we observe that both MD-OA and BD-OA reduce much more average loss than MD-OCO and MGD-OCO for a large $\sigma$, which validates our theoretical results nicely. It means that OA is more effective to reduce the switching cost than OCO for a large $\sigma$. Specifically, as shown in Figure \ref{figure_ave_loss_separate}, the average loss caused by switching cost of OA methods, i.e., MD-OA(SL), has unsignificant changes, but that of OCO methods, i.e., MD-OCO(SL), has remarkable increase for a large $\sigma$.

When handling the whole dataset, the final difference of switching cost between MD-OA and MD-OCO is shown in Figure \ref{figure_ave_loss_difference_switching_cost}. Here, the difference of switching cost is measured by using \textit{average loss caused by switching cost} of MD-OCO minus corresponding \textit{average loss caused by switching cost} of MD-OA. As we can see, it highlights that OA is more effective to decrease the switching cost. The superiority becomes significant for a large $\sigma$, which verifies our theoretical results nicely again.

\section{Conclusion and future work}
\label{sect_conclusion}
We have proposed a new  dynamic regret with switching cost and a new analysis framework for both online algorithms and online convex optimization.  We find that  the switching cost significantly impacts on the regret yielded by OA methods, but does not have an impact on the regret yielded by OCO methods. Empirical studies have validated our theoretical result.

Moreover, the switching cost in the paper is measured by using the norm of the difference between two successive decisions, that is, $\lrnorm{\x_{t+1} - \x_t}$. It is interesting to investigate whether the work can be extended to a more general distance measure function such as Bregman divergence $d_B(\x_{t+1}, \x_t)$ or Mahalanobis distance $d_M(\x_{t+1}, \x_t)$. Specifically, if the Bregman divergence\footnote{See details in \url{https://en.wikipedia.org/wiki/Bregman_divergence}.} is used, the switching cost is thus $d_B(\x_{t+1}, \x_t) = \psi(\x_{t+1}) - \psi(\x_t) - \lrangle{\nabla \psi(\x_t), \x_{t+1} - \x_t}$, where $\psi(\cdot)$ is a differentiable distance function.   If the Mahalanobis distance\footnote{See details in \url{https://en.wikipedia.org/wiki/Mahalanobis_distance}.} is used, the switching cost is thus $d_M(\x_{t+1}, \x_t) = \sqrt{(\x_{t+1} - \x_t)\Tr \S (\x_{t+1} - \x_t)}$, where $\S$ is the given covariance matrix. We leave the potential extension as the future work. 

Besides, our analysis provides regret bound for any given budget of dynamics $D$. It is a good direction to extend the work in the parameter-free setting, where analysis is adaptive to the dynamics $D$ of environment. Some previous work such as \cite{NIPS2018_7407} have proposed the adaptive online method and analysis framework. But, \cite{NIPS2018_7407} works in the expert setting, not a general setting of online convex optimization. It is still unknown whether their method can be used to extend our analysis.

\section{Acknowledgments}
This work was supported by the National Key R \& D Program of China 2018YFB1003203 and the National Natural Science Foundation of China (Grant No. 61672528, 61773392, and 61671463).

\bibliographystyle{ACM-Reference-Format}
\bibliography{reference}


\begin{thebibliography}{51}


\ifx \showCODEN    \undefined \def \showCODEN     #1{\unskip}     \fi
\ifx \showDOI      \undefined \def \showDOI       #1{#1}\fi
\ifx \showISBNx    \undefined \def \showISBNx     #1{\unskip}     \fi
\ifx \showISBNxiii \undefined \def \showISBNxiii  #1{\unskip}     \fi
\ifx \showISSN     \undefined \def \showISSN      #1{\unskip}     \fi
\ifx \showLCCN     \undefined \def \showLCCN      #1{\unskip}     \fi
\ifx \shownote     \undefined \def \shownote      #1{#1}          \fi
\ifx \showarticletitle \undefined \def \showarticletitle #1{#1}   \fi
\ifx \showURL      \undefined \def \showURL       {\relax}        \fi
\providecommand\bibfield[2]{#2}
\providecommand\bibinfo[2]{#2}
\providecommand\natexlab[1]{#1}
\providecommand\showeprint[2][]{arXiv:#2}

\bibitem[\protect\citeauthoryear{Abernethy, Bartlett, Buchbinder, and
  Stanton}{Abernethy et~al\mbox{.}}{2010}]%
        {Abernethy:2010}
\bibfield{author}{\bibinfo{person}{Jacob Abernethy}, \bibinfo{person}{Peter~L.
  Bartlett}, \bibinfo{person}{Niv Buchbinder}, {and} \bibinfo{person}{Isabelle
  Stanton}.} \bibinfo{year}{2010}\natexlab{}.
\newblock \showarticletitle{A Regularization Approach to Metrical Task
  Systems}. In \bibinfo{booktitle}{\emph{Proceedings of the 21st International
  Conference on Algorithmic Learning Theory (ALT)}}.
  \bibinfo{publisher}{Springer-Verlag}, \bibinfo{address}{Berlin, Heidelberg},
  \bibinfo{pages}{270--284}.
\newblock


\bibitem[\protect\citeauthoryear{Andrew, Barman, Ligett, Lin, Meyerson,
  Roytman, and Wierman}{Andrew et~al\mbox{.}}{2013}]%
        {Andrew:2013:TTM}
\bibfield{author}{\bibinfo{person}{Lachlan Andrew}, \bibinfo{person}{Siddharth
  Barman}, \bibinfo{person}{Katrina Ligett}, \bibinfo{person}{Minghong Lin},
  \bibinfo{person}{Adam Meyerson}, \bibinfo{person}{Alan Roytman}, {and}
  \bibinfo{person}{Adam Wierman}.} \bibinfo{year}{2013}\natexlab{}.
\newblock \showarticletitle{A Tale of Two Metrics: Simultaneous Bounds on
  Competitiveness and Regret}. In \bibinfo{booktitle}{\emph{Proceedings of the
  ACM International Conference on Measurement and Modeling of Computer
  Systems}}. \bibinfo{pages}{329--330}.
\newblock


\bibitem[\protect\citeauthoryear{Antoniadis, Schewior, and
  Fleischer}{Antoniadis et~al\mbox{.}}{2018}]%
        {Antoniadis-10.1007}
\bibfield{author}{\bibinfo{person}{Antonios Antoniadis}, \bibinfo{person}{Kevin
  Schewior}, {and} \bibinfo{person}{Rudolf Fleischer}.}
  \bibinfo{year}{2018}\natexlab{}.
\newblock \showarticletitle{A Tight Lower Bound for Online Convex Optimization
  with Switching Costs}. In \bibinfo{booktitle}{\emph{Approximation and Online
  Algorithms}}. \bibinfo{publisher}{Springer International Publishing},
  \bibinfo{address}{Cham}, \bibinfo{pages}{164--175}.
\newblock


\bibitem[\protect\citeauthoryear{Bansal, Buchbinder, and Naor}{Bansal
  et~al\mbox{.}}{2010}]%
        {Bansal:2010:MTS}
\bibfield{author}{\bibinfo{person}{Nikhil Bansal}, \bibinfo{person}{Niv
  Buchbinder}, {and} \bibinfo{person}{Joseph Naor}.}
  \bibinfo{year}{2010}\natexlab{}.
\newblock \showarticletitle{Metrical Task Systems and the K-server Problem on
  HSTs}. In \bibinfo{booktitle}{\emph{Proceedings of the 37th International
  Colloquium Conference on Automata, Languages and Programming}}.
\newblock


\bibitem[\protect\citeauthoryear{Beck and Teboulle}{Beck and Teboulle}{2003}]%
        {BECK2003167}
\bibfield{author}{\bibinfo{person}{Amir Beck} {and} \bibinfo{person}{Marc
  Teboulle}.} \bibinfo{year}{2003}\natexlab{}.
\newblock \showarticletitle{Mirror descent and nonlinear projected subgradient
  methods for convex optimization}.
\newblock \bibinfo{journal}{\emph{Operations Research Letters}}
  \bibinfo{volume}{31}, \bibinfo{number}{3} (\bibinfo{year}{2003}),
  \bibinfo{pages}{167 -- 175}.
\newblock


\bibitem[\protect\citeauthoryear{Bernstein, Mannor, and Shimkin}{Bernstein
  et~al\mbox{.}}{2010}]%
        {NIPS2010_3896}
\bibfield{author}{\bibinfo{person}{Andrey Bernstein}, \bibinfo{person}{Shie
  Mannor}, {and} \bibinfo{person}{Nahum Shimkin}.}
  \bibinfo{year}{2010}\natexlab{}.
\newblock \showarticletitle{Online Classification with Specificity
  Constraints}.
\newblock In \bibinfo{booktitle}{\emph{Proceedings of Advances in Neural
  Information Processing Systems (NIPS)}},
  \bibfield{editor}{\bibinfo{person}{J.~D. Lafferty}, \bibinfo{person}{C.~K.~I.
  Williams}, \bibinfo{person}{J.~Shawe-Taylor}, \bibinfo{person}{R.~S. Zemel},
  {and} \bibinfo{person}{A.~Culotta}} (Eds.). \bibinfo{pages}{190--198}.
\newblock


\bibitem[\protect\citeauthoryear{Besbes, Gur, and Zeevi}{Besbes
  et~al\mbox{.}}{2015}]%
        {Besbes:2015gb}
\bibfield{author}{\bibinfo{person}{Omar Besbes}, \bibinfo{person}{Yonatan Gur},
  {and} \bibinfo{person}{Assaf~J Zeevi}.} \bibinfo{year}{2015}\natexlab{}.
\newblock \showarticletitle{{Non-Stationary Stochastic Optimization.}}
\newblock \bibinfo{journal}{\emph{Operations Research}} \bibinfo{volume}{63},
  \bibinfo{number}{5} (\bibinfo{year}{2015}), \bibinfo{pages}{1227--1244}.
\newblock


\bibitem[\protect\citeauthoryear{Blum and Burch}{Blum and Burch}{2000}]%
        {Blum2000}
\bibfield{author}{\bibinfo{person}{Avrim Blum} {and} \bibinfo{person}{Carl
  Burch}.} \bibinfo{year}{2000}\natexlab{}.
\newblock \showarticletitle{On-line Learning and the Metrical Task System
  Problem}.
\newblock \bibinfo{journal}{\emph{Machine Learning}} \bibinfo{volume}{39},
  \bibinfo{number}{1} (\bibinfo{date}{Apr} \bibinfo{year}{2000}),
  \bibinfo{pages}{35--58}.
\newblock


\bibitem[\protect\citeauthoryear{Bubeck}{Bubeck}{2011}]%
        {introduction-online-optimization}
\bibfield{author}{\bibinfo{person}{Sébastien Bubeck}.}
  \bibinfo{year}{2011}\natexlab{}.
\newblock \bibinfo{title}{Introduction to Online Optimization}.
\newblock
\newblock


\bibitem[\protect\citeauthoryear{Bubeck, Cohen, Lee, and Lee}{Bubeck
  et~al\mbox{.}}{2019}]%
        {Bubeck:2019vp}
\bibfield{author}{\bibinfo{person}{S{\'e}bastien Bubeck},
  \bibinfo{person}{Michael~B Cohen}, \bibinfo{person}{James~R Lee}, {and}
  \bibinfo{person}{Yin~Tat Lee}.} \bibinfo{year}{2019}\natexlab{}.
\newblock \showarticletitle{{Metrical task systems on trees via mirror descent
  and unfair gluing}}. In \bibinfo{booktitle}{\emph{Proceedings of the ACM-SIAM
  Symposium on Discrete Algorithms (SODA)}}.
\newblock


\bibitem[\protect\citeauthoryear{Bubeck, Cohen, Lee, Lee, and M\k{a}dry}{Bubeck
  et~al\mbox{.}}{2018}]%
        {Bubeck:2018:KVM}
\bibfield{author}{\bibinfo{person}{S{\'e}bastien Bubeck},
  \bibinfo{person}{Michael~B. Cohen}, \bibinfo{person}{Yin~Tat Lee},
  \bibinfo{person}{James~R. Lee}, {and} \bibinfo{person}{Aleksander
  M\k{a}dry}.} \bibinfo{year}{2018}\natexlab{}.
\newblock \showarticletitle{K-server via Multiscale Entropic Regularization}.
  In \bibinfo{booktitle}{\emph{Proceedings of the 50th Annual ACM Symposium on
  Theory of Computing (STOC)}}. \bibinfo{publisher}{ACM}, \bibinfo{address}{New
  York, NY, USA}, \bibinfo{pages}{3--16}.
\newblock


\bibitem[\protect\citeauthoryear{Buchbinder, Chen, Naor, and Shamir}{Buchbinder
  et~al\mbox{.}}{2012}]%
        {pmlr-v23-buchbinder12}
\bibfield{author}{\bibinfo{person}{Niv Buchbinder}, \bibinfo{person}{Shahar
  Chen}, \bibinfo{person}{Joshep~(Seffi) Naor}, {and} \bibinfo{person}{Ohad
  Shamir}.} \bibinfo{year}{2012}\natexlab{}.
\newblock \showarticletitle{Unified Algorithms for Online Learning and
  Competitive Analysis}. In \bibinfo{booktitle}{\emph{Proceedings of the 25th
  Annual Conference on Learning Theory (COLT)}},
  \bibfield{editor}{\bibinfo{person}{Shie Mannor}, \bibinfo{person}{Nathan
  Srebro}, {and} \bibinfo{person}{Robert~C. Williamson}} (Eds.),
  Vol.~\bibinfo{volume}{23}. \bibinfo{address}{Edinburgh, Scotland},
  \bibinfo{pages}{5.1--5.18}.
\newblock


\bibitem[\protect\citeauthoryear{Chen, Agarwal, Wierman, Barman, and
  Andrew}{Chen et~al\mbox{.}}{2015}]%
        {Chen:2015:OCO}
\bibfield{author}{\bibinfo{person}{Niangjun Chen}, \bibinfo{person}{Anish
  Agarwal}, \bibinfo{person}{Adam Wierman}, \bibinfo{person}{Siddharth Barman},
  {and} \bibinfo{person}{Lachlan~L.H. Andrew}.}
  \bibinfo{year}{2015}\natexlab{}.
\newblock \showarticletitle{Online Convex Optimization Using Predictions}. In
  \bibinfo{booktitle}{\emph{Proceedings of the ACM International Conference on
  Measurement and Modeling of Computer Systems}}. \bibinfo{pages}{191--204}.
\newblock


\bibitem[\protect\citeauthoryear{Chen, Comden, Liu, Gandhi, and Wierman}{Chen
  et~al\mbox{.}}{2016}]%
        {Chen:2016:UPO}
\bibfield{author}{\bibinfo{person}{Niangjun Chen}, \bibinfo{person}{Joshua
  Comden}, \bibinfo{person}{Zhenhua Liu}, \bibinfo{person}{Anshul Gandhi},
  {and} \bibinfo{person}{Adam Wierman}.} \bibinfo{year}{2016}\natexlab{}.
\newblock \showarticletitle{Using Predictions in Online Optimization: Looking
  Forward with an Eye on the Past}. In \bibinfo{booktitle}{\emph{Proceedings of
  the 2016 ACM SIGMETRICS International Conference on Measurement and Modeling
  of Computer Science}}. \bibinfo{pages}{193--206}.
\newblock


\bibitem[\protect\citeauthoryear{Chen, Goel, and Wierman}{Chen
  et~al\mbox{.}}{2018}]%
        {pmlr-v75-chen18b}
\bibfield{author}{\bibinfo{person}{Niangjun Chen}, \bibinfo{person}{Gautam
  Goel}, {and} \bibinfo{person}{Adam Wierman}.}
  \bibinfo{year}{2018}\natexlab{}.
\newblock \showarticletitle{Smoothed Online Convex Optimization in High
  Dimensions via Online Balanced Descent}. In
  \bibinfo{booktitle}{\emph{Proceedings of the 31st Conference On Learning
  Theory (COLT)}}, Vol.~\bibinfo{volume}{75}. \bibinfo{pages}{1574--1594}.
\newblock


\bibitem[\protect\citeauthoryear{Chen, Ling, and Giannakis}{Chen
  et~al\mbox{.}}{2017}]%
        {Chen:2017tt}
\bibfield{author}{\bibinfo{person}{Tianyi Chen}, \bibinfo{person}{Qing Ling},
  {and} \bibinfo{person}{Georgios~B. Giannakis}.}
  \bibinfo{year}{2017}\natexlab{}.
\newblock \showarticletitle{An Online Convex Optimization Approach to Proactive
  Network Resource Allocation}.
\newblock \bibinfo{journal}{\emph{IEEE Transactions on Signal Processing}}
  \bibinfo{volume}{65} (\bibinfo{year}{2017}), \bibinfo{pages}{6350--6364}.
\newblock


\bibitem[\protect\citeauthoryear{Chiang, Yang, Lee, Mahdavi, Lu, Jin, and
  Zhu}{Chiang et~al\mbox{.}}{2012}]%
        {Chiang2012Online}
\bibfield{author}{\bibinfo{person}{Chao~Kai Chiang}, \bibinfo{person}{Tianbao
  Yang}, \bibinfo{person}{Chia~Jung Lee}, \bibinfo{person}{Mehrdad Mahdavi},
  \bibinfo{person}{Chi~Jen Lu}, \bibinfo{person}{Rong Jin}, {and}
  \bibinfo{person}{Shenghuo Zhu}.} \bibinfo{year}{2012}\natexlab{}.
\newblock \showarticletitle{Online Optimization with Gradual Variations}.
\newblock \bibinfo{journal}{\emph{Journal of Machine Learning Research}}
  \bibinfo{volume}{23} (\bibinfo{year}{2012}).
\newblock


\bibitem[\protect\citeauthoryear{Crammer, Kandola, and Singer}{Crammer
  et~al\mbox{.}}{2004}]%
        {NIPS2003_2385}
\bibfield{author}{\bibinfo{person}{Koby Crammer}, \bibinfo{person}{Jaz
  Kandola}, {and} \bibinfo{person}{Yoram Singer}.}
  \bibinfo{year}{2004}\natexlab{}.
\newblock \showarticletitle{Online Classification on a Budget}.
\newblock In \bibinfo{booktitle}{\emph{Proceedings of Advances in Neural
  Information Processing Systems (NIPS)}}. \bibinfo{pages}{225--232}.
\newblock


\bibitem[\protect\citeauthoryear{Gao, Li, and Zhang}{Gao et~al\mbox{.}}{2018}]%
        {pmlr-v84-gao18a}
\bibfield{author}{\bibinfo{person}{Xiand Gao}, \bibinfo{person}{Xiaobo Li},
  {and} \bibinfo{person}{Shuzhong Zhang}.} \bibinfo{year}{2018}\natexlab{}.
\newblock \showarticletitle{Online Learning with Non-Convex Losses and
  Non-Stationary Regret}. In \bibinfo{booktitle}{\emph{Proceedings of the
  Twenty-First International Conference on Artificial Intelligence and
  Statistics (AISTATS)}}, \bibfield{editor}{\bibinfo{person}{Amos Storkey}
  {and} \bibinfo{person}{Fernando Perez-Cruz}} (Eds.),
  Vol.~\bibinfo{volume}{84}. \bibinfo{pages}{235--243}.
\newblock


\bibitem[\protect\citeauthoryear{Gy\"{o}rgy and Szepesv\'{a}ri}{Gy\"{o}rgy and
  Szepesv\'{a}ri}{2016}]%
        {Gyorgy:2016}
\bibfield{author}{\bibinfo{person}{Andr\'{a}s Gy\"{o}rgy} {and}
  \bibinfo{person}{Csaba Szepesv\'{a}ri}.} \bibinfo{year}{2016}\natexlab{}.
\newblock \showarticletitle{Shifting Regret, Mirror Descent, and Matrices}. In
  \bibinfo{booktitle}{\emph{Proceedings of the 33rd International Conference on
  Machine Learning (ICML)}}. \bibinfo{publisher}{JMLR.org},
  \bibinfo{pages}{2943--2951}.
\newblock


\bibitem[\protect\citeauthoryear{Hall and Willett}{Hall and Willett}{2013}]%
        {Hall:2013vr}
\bibfield{author}{\bibinfo{person}{Eric~C Hall} {and} \bibinfo{person}{Rebecca
  Willett}.} \bibinfo{year}{2013}\natexlab{}.
\newblock \showarticletitle{{Dynamical Models and tracking regret in online
  convex programming.}}. In \bibinfo{booktitle}{\emph{Proceedings of
  International Conference on International Conference on Machine Learning
  (ICML)}}.
\newblock


\bibitem[\protect\citeauthoryear{Hall and Willett}{Hall and Willett}{2015}]%
        {Hall:2015ct}
\bibfield{author}{\bibinfo{person}{Eric~C Hall} {and}
  \bibinfo{person}{Rebecca~M Willett}.} \bibinfo{year}{2015}\natexlab{}.
\newblock \showarticletitle{{Online Convex Optimization in Dynamic
  Environments.}}
\newblock \bibinfo{journal}{\emph{IEEE Journal of Selected Topics in Signal
  Processing}} \bibinfo{volume}{9}, \bibinfo{number}{4} (\bibinfo{year}{2015}),
  \bibinfo{pages}{647--662}.
\newblock


\bibitem[\protect\citeauthoryear{Hazan}{Hazan}{2016}]%
        {Hazan2016Introduction}
\bibfield{author}{\bibinfo{person}{Elad Hazan}.}
  \bibinfo{year}{2016}\natexlab{}.
\newblock \showarticletitle{Introduction to Online Convex Optimization}.
\newblock \bibinfo{journal}{\emph{Foundations and Trends in Optimization}}
  \bibinfo{volume}{2}, \bibinfo{number}{3-4} (\bibinfo{year}{2016}),
  \bibinfo{pages}{157--325}.
\newblock


\bibitem[\protect\citeauthoryear{Jadbabaie, Rakhlin, Shahrampour, and
  Sridharan}{Jadbabaie et~al\mbox{.}}{2015}]%
        {Jadbabaie:2015wg}
\bibfield{author}{\bibinfo{person}{Ali Jadbabaie}, \bibinfo{person}{Alexander
  Rakhlin}, \bibinfo{person}{Shahin Shahrampour}, {and}
  \bibinfo{person}{Karthik Sridharan}.} \bibinfo{year}{2015}\natexlab{}.
\newblock \showarticletitle{{Online Optimization : Competing with Dynamic
  Comparators}}. In \bibinfo{booktitle}{\emph{Proceedings of International
  Conference on Artificial Intelligence and Statistics (AISTATS)}}.
  \bibinfo{pages}{398--406}.
\newblock


\bibitem[\protect\citeauthoryear{Jenatton, Huang, and Archambeau}{Jenatton
  et~al\mbox{.}}{2016}]%
        {pmlr-v48-jenatton16}
\bibfield{author}{\bibinfo{person}{Rodolphe Jenatton}, \bibinfo{person}{Jim
  Huang}, {and} \bibinfo{person}{Cedric Archambeau}.}
  \bibinfo{year}{2016}\natexlab{}.
\newblock \showarticletitle{Adaptive Algorithms for Online Convex Optimization
  with Long-term Constraints}. In \bibinfo{booktitle}{\emph{Proceedings of The
  33rd International Conference on Machine Learning (ICML)}},
  Vol.~\bibinfo{volume}{48}. \bibinfo{pages}{402--411}.
\newblock


\bibitem[\protect\citeauthoryear{Lee}{Lee}{2018}]%
        {Lee:2018wt}
\bibfield{author}{\bibinfo{person}{James~R Lee}.}
  \bibinfo{year}{2018}\natexlab{}.
\newblock \showarticletitle{{Fusible HSTs and the randomized k-server
  conjecture.}}. In \bibinfo{booktitle}{\emph{Proceedings of the IEEE 59th
  Annual Symposium on Foundations of Computer Science}}.
\newblock


\bibitem[\protect\citeauthoryear{Li and Hoi}{Li and Hoi}{2014}]%
        {Li:2014:OPS}
\bibfield{author}{\bibinfo{person}{Bin Li} {and} \bibinfo{person}{Steven C.~H.
  Hoi}.} \bibinfo{year}{2014}\natexlab{}.
\newblock \showarticletitle{Online Portfolio Selection: A Survey}.
\newblock \bibinfo{journal}{\emph{Comput. Surveys}} \bibinfo{volume}{46},
  \bibinfo{number}{3} (\bibinfo{year}{2014}), \bibinfo{pages}{35:1--35:36}.
\newblock


\bibitem[\protect\citeauthoryear{Li, Hoi, Zhao, and Gopalkrishnan}{Li
  et~al\mbox{.}}{2013}]%
        {Li:2013:CWM}
\bibfield{author}{\bibinfo{person}{Bin Li}, \bibinfo{person}{Steven C.~H. Hoi},
  \bibinfo{person}{Peilin Zhao}, {and} \bibinfo{person}{Vivekanand
  Gopalkrishnan}.} \bibinfo{year}{2013}\natexlab{}.
\newblock \showarticletitle{Confidence Weighted Mean Reversion Strategy for
  Online Portfolio Selection}.
\newblock \bibinfo{journal}{\emph{ACM Transactions on Knowledge Discovery from
  Data (TKDD)}} \bibinfo{volume}{7}, \bibinfo{number}{1} (\bibinfo{date}{March}
  \bibinfo{year}{2013}), \bibinfo{pages}{4:1--4:38}.
\newblock


\bibitem[\protect\citeauthoryear{{Li}, {Zhou}, {Xiong}, {Wang}, and
  {Wang}}{{Li} et~al\mbox{.}}{2018}]%
        {8260919}
\bibfield{author}{\bibinfo{person}{C. {Li}}, \bibinfo{person}{P. {Zhou}},
  \bibinfo{person}{L. {Xiong}}, \bibinfo{person}{Q. {Wang}}, {and}
  \bibinfo{person}{T. {Wang}}.} \bibinfo{year}{2018}\natexlab{}.
\newblock \showarticletitle{Differentially Private Distributed Online
  Learning}.
\newblock \bibinfo{journal}{\emph{IEEE Transactions on Knowledge and Data
  Engineering (TKDE)}} \bibinfo{volume}{30}, \bibinfo{number}{8}
  (\bibinfo{date}{Aug} \bibinfo{year}{2018}), \bibinfo{pages}{1440--1453}.
\newblock


\bibitem[\protect\citeauthoryear{Li, Qu, and Li}{Li et~al\mbox{.}}{2018}]%
        {Li:2018uy}
\bibfield{author}{\bibinfo{person}{Yingying Li}, \bibinfo{person}{Guannan Qu},
  {and} \bibinfo{person}{Na Li}.} \bibinfo{year}{2018}\natexlab{}.
\newblock \showarticletitle{{Online Optimization with Predictions and Switching
  Costs: Fast Algorithms and the Fundamental Limit}}.
\newblock \bibinfo{journal}{\emph{arXiv.org}} (\bibinfo{date}{Jan.}
  \bibinfo{year}{2018}).
\newblock
\showeprint[arxiv]{math.OC/1801.07780v3}


\bibitem[\protect\citeauthoryear{Lin, Wierman, Andrew, and Thereska}{Lin
  et~al\mbox{.}}{2011}]%
        {5934885}
\bibfield{author}{\bibinfo{person}{M. Lin}, \bibinfo{person}{A. Wierman},
  \bibinfo{person}{L.~L.~H. Andrew}, {and} \bibinfo{person}{E. Thereska}.}
  \bibinfo{year}{2011}\natexlab{}.
\newblock \showarticletitle{Dynamic right-sizing for power-proportional data
  centers}. In \bibinfo{booktitle}{\emph{Proceedings of IEEE International
  Conference on Computer Communications (INFOCOMM)}}.
  \bibinfo{pages}{1098--1106}.
\newblock


\bibitem[\protect\citeauthoryear{Lin, Wierman, Roytman, Meyerson, and
  Andrew}{Lin et~al\mbox{.}}{2012}]%
        {Lin:2012:OOS}
\bibfield{author}{\bibinfo{person}{Minghong Lin}, \bibinfo{person}{Adam
  Wierman}, \bibinfo{person}{Alan Roytman}, \bibinfo{person}{Adam Meyerson},
  {and} \bibinfo{person}{Lachlan~L.H. Andrew}.}
  \bibinfo{year}{2012}\natexlab{}.
\newblock \showarticletitle{Online Optimization with Switching Cost}.
\newblock \bibinfo{journal}{\emph{SIGMETRICS Performance Evaluation Review}}
  \bibinfo{volume}{40}, \bibinfo{number}{3} (\bibinfo{year}{2012}),
  \bibinfo{pages}{98--100}.
\newblock


\bibitem[\protect\citeauthoryear{Lu, Chen, and Andrew}{Lu
  et~al\mbox{.}}{2013}]%
        {6269026}
\bibfield{author}{\bibinfo{person}{T. Lu}, \bibinfo{person}{M. Chen}, {and}
  \bibinfo{person}{L.~L.~H. Andrew}.} \bibinfo{year}{2013}\natexlab{}.
\newblock \showarticletitle{Simple and Effective Dynamic Provisioning for
  Power-Proportional Data Centers}.
\newblock \bibinfo{journal}{\emph{IEEE Transactions on Parallel and Distributed
  Systems (TPDS)}} \bibinfo{volume}{24}, \bibinfo{number}{6}
  (\bibinfo{date}{June} \bibinfo{year}{2013}), \bibinfo{pages}{1161--1171}.
\newblock


\bibitem[\protect\citeauthoryear{Mokhtari, Shahrampour, Jadbabaie, and
  Ribeiro}{Mokhtari et~al\mbox{.}}{2016}]%
        {Mokhtari:2016jz}
\bibfield{author}{\bibinfo{person}{Aryan Mokhtari}, \bibinfo{person}{Shahin
  Shahrampour}, \bibinfo{person}{Ali Jadbabaie}, {and}
  \bibinfo{person}{Alejandro Ribeiro}.} \bibinfo{year}{2016}\natexlab{}.
\newblock \showarticletitle{{Online optimization in dynamic environments:
  Improved regret rates for strongly convex problems}}. In
  \bibinfo{booktitle}{\emph{Proceedings of IEEE Conference on Decision and
  Control (CDC)}}. \bibinfo{publisher}{IEEE}, \bibinfo{pages}{7195--7201}.
\newblock


\bibitem[\protect\citeauthoryear{Mookherjee, Hobbs, Friesz, and
  Rigdon}{Mookherjee et~al\mbox{.}}{2008}]%
        {2b1edbffc}
\bibfield{author}{\bibinfo{person}{Reetabrata Mookherjee},
  \bibinfo{person}{{Benjamin F.} Hobbs}, \bibinfo{person}{{Terry Lee} Friesz},
  {and} \bibinfo{person}{{Matthew A.} Rigdon}.}
  \bibinfo{year}{2008}\natexlab{}.
\newblock \showarticletitle{Dynamic oligopolistic competition on an electric
  power network with ramping costs and joint sales constraints}.
\newblock \bibinfo{journal}{\emph{Journal of Industrial and Management
  Optimization}} \bibinfo{volume}{4}, \bibinfo{number}{3} (\bibinfo{date}{11}
  \bibinfo{year}{2008}), \bibinfo{pages}{425--452}.
\newblock


\bibitem[\protect\citeauthoryear{Morari and Lee}{Morari and Lee}{1999}]%
        {MORARI1999667}
\bibfield{author}{\bibinfo{person}{Manfred Morari} {and}
  \bibinfo{person}{Jay~H. Lee}.} \bibinfo{year}{1999}\natexlab{}.
\newblock \showarticletitle{Model predictive control: past, present and
  future}.
\newblock \bibinfo{journal}{\emph{Computers \& Chemical Engineering}}
  \bibinfo{volume}{23}, \bibinfo{number}{4} (\bibinfo{year}{1999}),
  \bibinfo{pages}{667 -- 682}.
\newblock


\bibitem[\protect\citeauthoryear{Renault and Ros{\'e}n}{Renault and
  Ros{\'e}n}{2012}]%
        {renault-101007}
\bibfield{author}{\bibinfo{person}{Marc~P. Renault} {and} \bibinfo{person}{Adi
  Ros{\'e}n}.} \bibinfo{year}{2012}\natexlab{}.
\newblock \showarticletitle{On Online Algorithms with Advice for the k-Server
  Problem}. In \bibinfo{booktitle}{\emph{Approximation and Online Algorithms}},
  \bibfield{editor}{\bibinfo{person}{Roberto Solis-Oba} {and}
  \bibinfo{person}{Giuseppe Persiano}} (Eds.). \bibinfo{publisher}{Springer
  Berlin Heidelberg}, \bibinfo{address}{Berlin, Heidelberg},
  \bibinfo{pages}{198--210}.
\newblock


\bibitem[\protect\citeauthoryear{Shalev-Shwartz}{Shalev-Shwartz}{2012}]%
        {ShalevShwartz:2012dz}
\bibfield{author}{\bibinfo{person}{Shai Shalev-Shwartz}.}
  \bibinfo{year}{2012}\natexlab{}.
\newblock \showarticletitle{{Online Learning and Online Convex Optimization}}.
\newblock \bibinfo{journal}{\emph{Foundations and Trends{\textregistered} in
  Machine Learning}} \bibinfo{volume}{4}, \bibinfo{number}{2}
  (\bibinfo{year}{2012}), \bibinfo{pages}{107--194}.
\newblock


\bibitem[\protect\citeauthoryear{{Sun}, {Tang}, {Minku}, {Wang}, and
  {Yao}}{{Sun} et~al\mbox{.}}{2016}]%
        {7401075}
\bibfield{author}{\bibinfo{person}{Y. {Sun}}, \bibinfo{person}{K. {Tang}},
  \bibinfo{person}{L.~L. {Minku}}, \bibinfo{person}{S. {Wang}}, {and}
  \bibinfo{person}{X. {Yao}}.} \bibinfo{year}{2016}\natexlab{}.
\newblock \showarticletitle{Online Ensemble Learning of Data Streams with
  Gradually Evolved Classes}.
\newblock \bibinfo{journal}{\emph{IEEE Transactions on Knowledge and Data
  Engineering (TKDE)}} \bibinfo{volume}{28}, \bibinfo{number}{6}
  (\bibinfo{date}{June} \bibinfo{year}{2016}), \bibinfo{pages}{1532--1545}.
\newblock


\bibitem[\protect\citeauthoryear{Wang, Huang, Lin, and Mohsenian-Rad}{Wang
  et~al\mbox{.}}{2014}]%
        {Wang:2014:ESG:2567529.2567556}
\bibfield{author}{\bibinfo{person}{Hao Wang}, \bibinfo{person}{Jianwei Huang},
  \bibinfo{person}{Xiaojun Lin}, {and} \bibinfo{person}{Hamed Mohsenian-Rad}.}
  \bibinfo{year}{2014}\natexlab{}.
\newblock \showarticletitle{Exploring Smart Grid and Data Center Interactions
  for Electric Power Load Balancing}.
\newblock \bibinfo{journal}{\emph{SIGMETRICS Performance Evaluation Review}}
  \bibinfo{volume}{41}, \bibinfo{number}{3} (\bibinfo{date}{Jan.}
  \bibinfo{year}{2014}), \bibinfo{pages}{89--94}.
\newblock


\bibitem[\protect\citeauthoryear{Wang, Lee, and Lu}{Wang et~al\mbox{.}}{2016}]%
        {Wang:2016:IAR}
\bibfield{author}{\bibinfo{person}{Liang Wang}, \bibinfo{person}{Kuang-chih
  Lee}, {and} \bibinfo{person}{Quan Lu}.} \bibinfo{year}{2016}\natexlab{}.
\newblock \showarticletitle{Improving Advertisement Recommendation by Enriching
  User Browser Cookie Attributes}. In \bibinfo{booktitle}{\emph{Proceedings of
  the 25th ACM International on Conference on Information and Knowledge
  Management (CIKM)}}. \bibinfo{pages}{2401--2404}.
\newblock


\bibitem[\protect\citeauthoryear{{Wang}, {Xu}, {Chen}, {Hao}, {Zhong}, and
  {Yu}}{{Wang} et~al\mbox{.}}{2019}]%
        {8610120}
\bibfield{author}{\bibinfo{person}{M. {Wang}}, \bibinfo{person}{C. {Xu}},
  \bibinfo{person}{X. {Chen}}, \bibinfo{person}{H. {Hao}}, \bibinfo{person}{L.
  {Zhong}}, {and} \bibinfo{person}{S. {Yu}}.} \bibinfo{year}{2019}\natexlab{}.
\newblock \showarticletitle{Differential Privacy Oriented Distributed Online
  Learning for Mobile Social Video Prefetching}.
\newblock \bibinfo{journal}{\emph{IEEE Transactions on Multimedia}}
  \bibinfo{volume}{21}, \bibinfo{number}{3} (\bibinfo{date}{March}
  \bibinfo{year}{2019}), \bibinfo{pages}{636--651}.
\newblock


\bibitem[\protect\citeauthoryear{Yang, Lyu, and King}{Yang
  et~al\mbox{.}}{2013}]%
        {Yang:2013:EOL}
\bibfield{author}{\bibinfo{person}{Haiqin Yang}, \bibinfo{person}{Michael~R.
  Lyu}, {and} \bibinfo{person}{Irwin King}.} \bibinfo{year}{2013}\natexlab{}.
\newblock \showarticletitle{Efficient Online Learning for Multitask Feature
  Selection}.
\newblock \bibinfo{journal}{\emph{ACM Transactions on Knowledge Discovery from
  Data (TKDD)}} \bibinfo{volume}{7}, \bibinfo{number}{2} (\bibinfo{date}{Aug.}
  \bibinfo{year}{2013}), \bibinfo{pages}{6:1--6:27}.
\newblock


\bibitem[\protect\citeauthoryear{Yang, Zhang, Jin, and Yi}{Yang
  et~al\mbox{.}}{2016}]%
        {Yang:2016ud}
\bibfield{author}{\bibinfo{person}{Tianbao Yang}, \bibinfo{person}{Lijun
  Zhang}, \bibinfo{person}{Rong Jin}, {and} \bibinfo{person}{Jinfeng Yi}.}
  \bibinfo{year}{2016}\natexlab{}.
\newblock \showarticletitle{{Tracking Slowly Moving Clairvoyant - Optimal
  Dynamic Regret of Online Learning with True and Noisy Gradient.}}. In
  \bibinfo{booktitle}{\emph{Proceedings of the 34th International Conference on
  Machine Learning (ICML)}}.
\newblock


\bibitem[\protect\citeauthoryear{Zhang, Lu, and Zhou}{Zhang
  et~al\mbox{.}}{2018a}]%
        {NIPS2018_7407}
\bibfield{author}{\bibinfo{person}{Lijun Zhang}, \bibinfo{person}{Shiyin Lu},
  {and} \bibinfo{person}{Zhi-Hua Zhou}.} \bibinfo{year}{2018}\natexlab{a}.
\newblock \showarticletitle{Adaptive Online Learning in Dynamic Environments}.
\newblock In \bibinfo{booktitle}{\emph{Advances in Neural Information
  Processing Systems 31}}, \bibfield{editor}{\bibinfo{person}{S.~Bengio},
  \bibinfo{person}{H.~Wallach}, \bibinfo{person}{H.~Larochelle},
  \bibinfo{person}{K.~Grauman}, \bibinfo{person}{N.~Cesa-Bianchi}, {and}
  \bibinfo{person}{R.~Garnett}} (Eds.). \bibinfo{pages}{1323--1333}.
\newblock


\bibitem[\protect\citeauthoryear{Zhang, Yang, rong jin, and Zhou}{Zhang
  et~al\mbox{.}}{2018b}]%
        {Zhang:2018tu}
\bibfield{author}{\bibinfo{person}{Lijun Zhang}, \bibinfo{person}{Tianbao
  Yang}, \bibinfo{person}{rong jin}, {and} \bibinfo{person}{Zhi-Hua Zhou}.}
  \bibinfo{year}{2018}\natexlab{b}.
\newblock \showarticletitle{Dynamic Regret of Strongly Adaptive Methods}. In
  \bibinfo{booktitle}{\emph{Proceedings of the 35th International Conference on
  Machine Learning (ICML)}}. \bibinfo{pages}{5882--5891}.
\newblock


\bibitem[\protect\citeauthoryear{Zhang, Yang, Yi, Jin, and Zhou}{Zhang
  et~al\mbox{.}}{2017a}]%
        {Zhang:2016wl}
\bibfield{author}{\bibinfo{person}{Lijun Zhang}, \bibinfo{person}{Tianbao
  Yang}, \bibinfo{person}{Jinfeng Yi}, \bibinfo{person}{Rong Jin}, {and}
  \bibinfo{person}{Zhi-Hua Zhou}.} \bibinfo{year}{2017}\natexlab{a}.
\newblock \showarticletitle{{Improved Dynamic Regret for Non-degenerate
  Functions}}. In \bibinfo{booktitle}{\emph{Proceedings of Neural Information
  Processing Systems (NIPS)}}.
\newblock


\bibitem[\protect\citeauthoryear{Zhang, Yangt, Yi, Jin, and Zhou}{Zhang
  et~al\mbox{.}}{2017b}]%
        {Zhang:2017:IDR}
\bibfield{author}{\bibinfo{person}{Lijun Zhang}, \bibinfo{person}{Tianbao
  Yangt}, \bibinfo{person}{Jinfeng Yi}, \bibinfo{person}{Rong Jin}, {and}
  \bibinfo{person}{Zhi-Hua Zhou}.} \bibinfo{year}{2017}\natexlab{b}.
\newblock \showarticletitle{Improved Dynamic Regret for Non-degenerate
  Functions}. In \bibinfo{booktitle}{\emph{Proceedings of the 31st
  International Conference on Neural Information Processing Systems}}.
  \bibinfo{pages}{732--741}.
\newblock


\bibitem[\protect\citeauthoryear{Zhang, Zhu, Zhani, and Boutaba}{Zhang
  et~al\mbox{.}}{2012}]%
        {6258025:zhang}
\bibfield{author}{\bibinfo{person}{Q. Zhang}, \bibinfo{person}{Q. Zhu},
  \bibinfo{person}{M.~F. Zhani}, {and} \bibinfo{person}{R. Boutaba}.}
  \bibinfo{year}{2012}\natexlab{}.
\newblock \showarticletitle{Dynamic Service Placement in Geographically
  Distributed Clouds}. In \bibinfo{booktitle}{\emph{Proceedings of the IEEE
  32nd International Conference on Distributed Computing Systems (ICDCS)}}.
  \bibinfo{pages}{526--535}.
\newblock


\bibitem[\protect\citeauthoryear{Zhao, Qiu, and Liu}{Zhao
  et~al\mbox{.}}{2018}]%
        {Zhao:2018wx}
\bibfield{author}{\bibinfo{person}{Yawei Zhao}, \bibinfo{person}{Shuang Qiu},
  {and} \bibinfo{person}{Ji Liu}.} \bibinfo{year}{2018}\natexlab{}.
\newblock \showarticletitle{{Proximal Online Gradient is Optimum for Dynamic
  Regret.}}
\newblock \bibinfo{journal}{\emph{CoRR}}  \bibinfo{volume}{cs.LG}
  (\bibinfo{year}{2018}).
\newblock


\bibitem[\protect\citeauthoryear{Zinkevich}{Zinkevich}{2003}]%
        {Zinkevich:2003}
\bibfield{author}{\bibinfo{person}{Martin Zinkevich}.}
  \bibinfo{year}{2003}\natexlab{}.
\newblock \showarticletitle{Online Convex Programming and Generalized
  Infinitesimal Gradient Ascent}. In \bibinfo{booktitle}{\emph{Proceedings of
  International Conference on Machine Learning (ICML)}}.
  \bibinfo{pages}{928--935}.
\newblock


\end{thebibliography}

\appendix
\section*{Proofs}
\begin{Lemma}
\label{lemma_mirror_descent_update_rule}
\nonumber
Given any vectors $\g$, $\u_t\in\Xcal$, $\u^\ast\in\Xcal$ , and a constant scalar $\lambda>0$, if 
\begin{align}
\nonumber
\u_{t+1} = \argmin_{\u\in\Xcal} \lrangle{\g, \u - \u_t} + \frac{1}{\lambda} B_{\Phi}(\u, \u_t),
\end{align} we have
\begin{align}
\nonumber
\lrangle{\g, \u_{t+1} - \u^\ast} \le \frac{1}{\lambda}\lrincir{ B_{\Phi}(\u^\ast, \u_t) -  B_{\Phi}(\u^\ast, \u_{t+1}) - B_{\Phi}(\u_{t+1}, \u_t) }.
\end{align}
\end{Lemma}
\begin{proof}

Denote $h(\u) = \lrangle{\g, \u-\u_t} + \frac{1}{\lambda}B_{\Phi}(\u, \u_t)$, and $\u_{\tau} = \u_{t+1} + \tau (\u^\ast - \u_{t+1})$. According to the optimality of $\x_t$, we have
\begin{align}
\nonumber
&0 \le  h(\u_\tau) - h(\u_{t+1}) \\ \nonumber
= & \lrangle{\g, \u_\tau - \u_{t+1}} + \frac{1}{\lambda}\lrincir{B_{\Phi}(\u_\tau, \u_t) - B_{\Phi}(\u_{t+1}, \u_t)} \\ \nonumber
= & \lrangle{\g, \tau (\u^\ast - \u_{t+1})} + \frac{1}{\lambda}\lrincir{ \Phi(\u_\tau) - \Phi(\u_{t+1})  + \lrangle{\nabla \Phi(\u_t), \tau (\u_{t+1} - \u^\ast)} } \\ \nonumber
\le & \lrangle{\g, \tau (\u^\ast - \u_{t+1})} + \frac{1}{\lambda} \lrangle{\nabla \Phi(\u_{t+1}), \tau (\u^\ast - \u_{t+1})} + \frac{1}{\lambda} \lrangle{\nabla \Phi(\u_t), \tau (\u_{t+1} - \u^\ast)}  \\ \nonumber
= & \lrangle{\g, \tau (\u^\ast - \u_{t+1})} + \frac{1}{\lambda} \lrangle{\nabla \Phi(\u_t)-\Phi(\u_{t+1}), \tau (\u_{t+1} - \u^\ast)}.
\end{align} Thus, we have
\begin{align}
\nonumber
& \lrangle{\g, \u_{t+1} - \u^\ast} \le \frac{1}{\lambda} \lrangle{\nabla \Phi(\u_t)-\Phi(\u_{t+1}), \u_{t+1} - \u^\ast}  \\ \nonumber
= & \frac{1}{\lambda}\lrincir{ B_{\Phi}(\u^\ast, \u_t) -  B_{\Phi}(\u^\ast, \u_{t+1}) - B_{\Phi}(\u_{t+1}, \u_t) }.
\end{align} It completes the proof.
\end{proof}

\begin{Lemma}
\label{lemma_dynamic_regret_bound}
For any $\x\in\Xcal$, we have
\begin{align}
\label{equa_theorem_upper_bound_oa_temp2}
B_\Phi(\y_{t+1}^\ast, \x) - B_{\Phi}(\y_t^\ast, \x) \le  2G \lrnorm{\y_{t+1}^\ast - \y_t^\ast}.
\end{align}
\end{Lemma}
\begin{proof}

According to the third-point identity of the Bregman divergence, we have
\begin{align}
\nonumber
&  B_\Phi(\y_{t+1}^\ast, \x) - B_{\Phi}(\y_t^\ast, \x) \\ \nonumber
= & \lrangle{\nabla \Phi(\y_{t+1}^\ast) - \nabla \Phi(\x), \y_{t+1}^\ast - \y_t^\ast} - B_{\Phi}(\y_t^\ast, \y_{t+1}^\ast) \\ \nonumber
\refabovecir{\le}{\textcircled{1}} &  \lrangle{\nabla \Phi(\y_{t+1}^\ast) - \nabla \Phi(\x), \y_{t+1}^\ast - \y_t^\ast} \\ \nonumber
\le & \lrnorm{\nabla \Phi(\y_{t+1}^\ast) - \nabla \Phi(\x)} \lrnorm{ \y_{t+1}^\ast - \y_t^\ast} \\ \nonumber
\le & \lrincir{\lrnorm{\nabla \Phi(\y_{t+1}^\ast)} + \lrnorm{ \nabla \Phi(\x)}} \lrnorm{ \y_{t+1}^\ast - \y_t^\ast} \\ \label{equa_theorem_upper_bound_oa_temp2}
\le & 2G \lrnorm{\y_{t+1}^\ast - \y_t^\ast}.
\end{align} $\textcircled{1}$ holds because $B_{\Phi}(\u,\v) \ge 0$ holds for any vectors $\u$ and $\v$. It completes the proof.
\end{proof}

\begin{Lemma}
\label{lemma_distance_between_x_oa}
Given $\x_{t-1}\in\Xcal$ and $\hat{\g}_t$, if $\x_t = \argmin_{\x \in \Xcal} \lrangle{\hat{\g}_t, \x - \x_{t-1}} + \frac{1}{\gamma}B_{\Phi}(\x, \x_{t-1})$, we have
\begin{align}
\nonumber
\lrnorm{\x_t - \x_{t-1}} \le \frac{2 G \gamma}{\mu}.
\end{align}
\end{Lemma}
\begin{proof}

\begin{align}
\nonumber
\lrangle{\hat{\g}_t, \x_t - \x_{t-1}} + \frac{\mu}{2\gamma} \lrnorm{\x_t - \x_{t-1}}^2 \refabovecir{\le}{\textcircled{1}} & \lrangle{\hat{\g}_t, \x_t - \x_{t-1}} + \frac{1}{\gamma} B_{\Phi}(\x_t, \x_{t-1}) \refabovecir{\le}{\textcircled{2}}  0.
\end{align} $\textcircled{1}$ holds due to $\Phi$ is $\mu$-strongly convex, and $\textcircled{2}$ holds due to the optimality of $\x_t$. Thus, 
\begin{align}
\nonumber
\frac{\mu}{2\gamma} \lrnorm{\x_t - \x_{t-1}}^2 \le \lrangle{\hat{\g}_t, -\x_t + \x_{t-1}} \le \lrnorm{\hat{\g}_t} \lrnorm{ -\x_t + \x_{t-1}} \le G \lrnorm{ -\x_t + \x_{t-1}}.
\end{align} That is, 
\begin{align}
\nonumber
\lrnorm{\x_t - \x_{t-1}} \le \frac{2 G \gamma}{\mu}.
\end{align}
It completes the proof.
\end{proof}

\textbf{Proof to Theorem \ref{theorem_regret_oa_upper_bound}:}
\begin{proof}

\begin{align}
\nonumber
& f_t(\x_t) - f_t(\y_t^{\ast}) \\ \nonumber
= & f_t(\x_t) - f_t(\x_{t-1}) + f_t(\x_{t-1}) - f_t(\y_t^{\ast}) \\ \nonumber
\le & f_t(\x_t) - f_t(\x_{t-1}) +\lrangle{\hat{\g}_t, \x_{t-1} - \y_t^{\ast}} \\ \nonumber
= & f_t(\x_t) - f_t(\x_{t-1}) - \lrangle{\hat{\g}_t, \x_t - \x_{t-1}} +  \lrangle{\hat{\g}_t, \x_t - \y_t^{\ast}} \\ \nonumber
\refabovecir{\le}{\textcircled{1}} & \frac{L}{2}\lrnorm{\x_{t-1} - \x_t}^2 +  \lrangle{\hat{\g}_t, \x_t - \y_t^{\ast}} \\ \nonumber
\refabovecir{\le}{\textcircled{2}} & \frac{L}{2}\lrnorm{\x_{t-1} - \x_t}^2 + \frac{1}{\gamma}\lrincir{ B_{\Phi}(\y_t^\ast, \x_{t-1}) -  B_{\Phi}(\y_t^\ast, \x_t) - B_{\Phi}(\x_t, \x_{t-1}) } \\ \nonumber
\refabovecir{\le}{\textcircled{3}} & \frac{L\gamma-\mu}{2\gamma}\lrnorm{\x_{t-1} - \x_t}^2 + \frac{1}{\gamma}\lrincir{ B_{\Phi}(\y_t^\ast, \x_{t-1}) -  B_{\Phi}(\y_t^\ast, \x_t)} \\ \label{equa_theorem_upper_bound_oa_temp1}
\refabovecir{\le}{\textcircled{4}} & \frac{1}{\gamma}\lrincir{ B_{\Phi}(\y_t^\ast, \x_{t-1}) -  B_{\Phi}(\y_t^\ast, \x_t)}.
\end{align} $\textcircled{1}$ holds because $f_t$ has $L$-Lipschitz gradient. $\textcircled{2}$ holds due to Lemma \ref{lemma_mirror_descent_update_rule} by setting $\g = \hat{\g}_t$, $\u_t = \x_{t-1}$, $\u_{t+1} = \x_t$, $\u^\ast = \y_t^\ast$, and $\lambda = \gamma$. $\textcircled{3}$ holds because that $\Phi$ is $\mu$-strongly convex, that is, $B_{\Phi}(\x_t, \x_{t-1}) \ge \frac{\mu}{2}\lrnorm{\x_t - \x_{t-1}}^2$. $\textcircled{4}$ holds due to $\gamma \le \frac{\mu}{L}$.

Thus, we have
\begin{align}
\nonumber
&\sum_{t=1}^T \lrincir{f_t(\x_t) - f_t(\y_t^{\ast}) + \lrnorm{\x_t - \x_{t-1}}^{\sigma} } - \sum_{t=1}^T \lrnorm{\y_t^\ast - \y_{t-1}^\ast}^{\sigma}  \\ \nonumber 
\le &\sum_{t=1}^T \lrincir{f_t(\x_t) - f_t(\y_t^{\ast}) + \lrnorm{\x_t - \x_{t-1}}^{\sigma} } \\ \nonumber
\refabovecir{\le}{\textcircled{1}} & \sum_{t=1}^T\lrnorm{\x_t - \x_{t-1}}^{\sigma} + \frac{1}{\gamma}\sum_{t=1}^T\lrincir{ B_{\Phi}(\y_t^\ast, \x_{t-1}) -  B_{\Phi}(\y_t^\ast, \x_t)} \\ \nonumber
= & \sum_{t=1}^T\lrnorm{\x_t - \x_{t-1}}^{\sigma} + \frac{1}{\gamma}\lrincir{B_{\Phi}(\y_1^\ast, \x_0) - B_{\Phi}(\y_T^\ast, \x_T)} + \frac{1}{\gamma}\sum_{t=1}^{T-1}\lrincir{ B_{\Phi}(\y_{t+1}^\ast, \x_t) -  B_{\Phi}(\y_t^\ast, \x_t)} \\ \nonumber
\refabovecir{\le}{\textcircled{2}} & \sum_{t=1}^T\lrnorm{\x_t - \x_{t-1}}^{\sigma} + \frac{2G}{\gamma}\sum_{t=1}^{T-1}\lrnorm{\y_{t+1}^{\ast} - \y_t^{\ast}} + \frac{1}{\gamma}\lrincir{B_{\Phi}(\y_1^{\ast}, \x_0) - B_{\Phi}(\y_T^{\ast}, \x_T)} \\ \nonumber
\le & \sum_{t=1}^T\lrnorm{\x_t - \x_{t-1}}^{\sigma} + \frac{2G}{\gamma}\sum_{t=1}^{T-1}\lrnorm{\y_{t+1}^{\ast} - \y_t^{\ast}} + \frac{1}{\gamma} B_{\Phi}(\y_1^{\ast}, \x_0) \\ \nonumber
\le & \sum_{t=1}^T\lrnorm{\x_t - \x_{t-1}}^{\sigma} + \frac{2GD}{\gamma} + \frac{R^2}{\gamma} \\ \nonumber
\refabovecir{\le}{\textcircled{3}} & \lrincir{\frac{2 G }{\mu}}^{\sigma} \gamma^{\sigma} T + \frac{2GD + R^2}{\gamma}.
\end{align} $\textcircled{1}$ holds due to \eqref{equa_theorem_upper_bound_oa_temp1}. 
$\textcircled{2}$ holds due to
\begin{align}
\nonumber
B_\Phi(\y_{t+1}^\ast, \x_t) - B_{\Phi}(\y_t^\ast, \x_t) \le  2G \lrnorm{\y_{t+1}^\ast - \y_t^\ast}
\end{align} according to Lemma \ref{lemma_dynamic_regret_bound}. $\textcircled{3}$ holds due to Lemma \ref{lemma_distance_between_x_oa}.

Choose $\gamma = \min \left\{ \frac{\mu}{L}, T^{-\frac{1}{1+\sigma}} D^{\frac{1}{1+\sigma}} \right\}$. We have
\begin{align}
\nonumber
&\sum_{t=1}^T \lrincir{f_t(\x_t) - f_t(\y_t^{\ast}) + \lrnorm{\x_t - \x_{t-1}}^{\sigma} } - \sum_{t=1}^T \lrnorm{\y_t^\ast - \y_{t-1}^\ast}^{\sigma}  \\ \nonumber 
\le  & \lrincir{\frac{2 G }{\mu}}^{\sigma} T^{\frac{1}{\sigma+1}}D^{\frac{\sigma}{\sigma+1}} + \max\left\{  \frac{L(2GD + R^2)}{\mu}, T^{\frac{1}{\sigma+1}} \lrincir{2G D^{\frac{\sigma}{\sigma+1}} + R^2 D^{-\frac{1}{\sigma+1}} } \right\}\\ \nonumber
\lesssim & T^{\frac{1}{\sigma+1}}D^{\frac{\sigma}{\sigma+1}} + T^{\frac{1}{\sigma+1}} D^{-\frac{1}{\sigma+1}}.
\end{align} 

Since it holds for any seqence $\{f_t\}_{t=1}^T \in \Fcal^T$, we finally obtain
\begin{align}
\nonumber
\sup_{\{f_t\}_{t=1}^T \in \Fcal^T}\Rcal_D^{\textsc{MD-OA}}
\lesssim & T^{\frac{1}{\sigma+1}}D^{\frac{\sigma}{\sigma+1}} + T^{\frac{1}{\sigma+1}} D^{-\frac{1}{\sigma+1}}.
\end{align} 
It completes the proof.
\end{proof}

\textbf{Proof to Theorem \ref{theorem_regret_oco_upper_bound}:} 
\begin{proof}

\begin{align}
\nonumber
& f_t(\x_t) -f_t(\y_t^\ast) + \lrnorm{\x_t - \x_{t+1}}^{\sigma} - \lrnorm{\y_t^\ast - \y_{t+1}^\ast}^{\sigma} \\ \nonumber 
\le & \lrangle{\bar{\g}_t, \x_t - \y_t^\ast } + \lrnorm{\x_t - \x_{t+1}}^{\sigma} \\ \nonumber
= & \lrangle{\bar{\g}_t, \x_t - \x_{t+1}} + \lrangle{\bar{\g}_t, \x_{t+1} - \y_t^\ast } + \lrnorm{\x_t - \x_{t+1}}^{\sigma}\\ \nonumber
\refabovecir{\le}{\textcircled{1}} & \lrangle{\bar{\g}_t, \x_t - \x_{t+1}} + \frac{1}{\eta}\lrincir{ B_{\Phi}(\y_t^\ast, \x_t) -  B_{\Phi}(\y_t^\ast, \x_{t+1}) - B_{\Phi}(\x_{t+1}, \x_t) } + \lrnorm{\x_t - \x_{t+1}}^{\sigma}\\ \nonumber
\refabovecir{\le}{\textcircled{2}} & \lrangle{\bar{\g}_t, \x_t - \x_{t+1}}  - \frac{\mu}{2\eta} \lrnorm{\x_{t+1} - \x_t}^2 + \frac{1}{\eta}\lrincir{ B_{\Phi}(\y_t^\ast, \x_t) -  B_{\Phi}(\y_t^\ast, \x_{t+1}) }  + \lrnorm{\x_t - \x_{t+1}}^{\sigma} \\ \nonumber
\refabovecir{\le}{\textcircled{3}} & \frac{\eta}{\mu}\lrnorm{\bar{\g}_t}^2 + \lrincir{ -\frac{\mu}{4\eta}\lrnorm{\x_{t+1}-\x_t}^2 + \lrnorm{\x_{t+1}-\x_t}^\sigma } + \frac{1}{\eta}\lrincir{ B_{\Phi}(\y_t^\ast, \x_t) -  B_{\Phi}(\y_t^\ast, \x_{t+1}) } \\ \nonumber
\le & \frac{\eta G^2}{\mu} + \lrincir{ -\lrincir{\frac{\sigma}{2}}^{\frac{2}{2-\sigma}}\lrincir{\frac{4\eta}{\mu}}^{\frac{\sigma}{2-\sigma}} + \lrincir{\frac{\sigma}{2}}^{\frac{\sigma}{2-\sigma}}\lrincir{\frac{4\eta}{\mu}}^{\frac{\sigma}{2-\sigma}} } + \frac{1}{\eta}\lrincir{ B_{\Phi}(\y_t^\ast, \x_t) -  B_{\Phi}(\y_t^\ast, \x_{t+1}) } \\ \nonumber
\le & \frac{\eta G^2}{\mu} + \lrincir{\frac{\sigma}{2}}^{\frac{\sigma}{2-\sigma}}\lrincir{\frac{4\eta}{\mu}}  + \frac{1}{\eta}\lrincir{ B_{\Phi}(\y_t^\ast, \x_t) -  B_{\Phi}(\y_t^\ast, \x_{t+1}) }.
\end{align} $\textcircled{1}$ holds due to Lemma \ref{lemma_mirror_descent_update_rule} by setting $\g = \bar{\g}_t$, $\u_t = \x_t$, $\u_{t+1} = \x_{t+1}$, $\u^\ast = \y_t^\ast$, and $\lambda = \eta$. $\textcircled{2}$ holds due to $\Phi$ is $\mu$-strongly convex.  $\textcircled{3}$  holds because $\lrangle{\u,\v} \le \frac{a}{2} \lrnorm{\u}^2 + \frac{1}{2a}\lrnorm{\v}^2$ holds for any $\u$, $\v$, and $a>0$. The last inequality holds due to $\eta\le\frac{\mu}{4}$ and $1\le \sigma \le 2$.

Telescoping it over $t$, we have
\begin{align}
\nonumber
& \sum_{t=1}^T \lrincir{f_t(\x_t) -f_t(\y_t^\ast) } + \sum_{t=1}^{T-1} \lrincir{ \lrnorm{\x_t - \x_{t+1}}^{\sigma} - \lrnorm{\y_t^\ast - \y_{t+1}^\ast}^{\sigma} } \\ \nonumber
\le & \frac{ T \eta G^2}{\mu} + \frac{1}{\eta} \sum_{t=1}^T \lrincir{ B_{\Phi}(\y_t^\ast, \x_t) -  B_{\Phi}(\y_t^\ast, \x_{t+1}) } + \lrincir{\frac{\sigma}{2}}^{\frac{\sigma}{2-\sigma}}\lrincir{\frac{4\eta}{\mu}} \\ \nonumber 
= & \frac{ T \eta G^2}{\mu} + \frac{1}{\eta} \lrincir{\sum_{t=2}^T \lrincir{ B_{\Phi}(\y_t^\ast, \x_t) -  B_{\Phi}(\y_{t-1}^\ast, \x_t) } } + \frac{1}{\eta} \lrincir{ B_{\Phi}(\y_1^\ast, \x_1) - B_{\Phi}(\y_T^\ast, \x_{T+1}) }\\ \nonumber
&  + \lrincir{\frac{\sigma}{2}}^{\frac{\sigma}{2-\sigma}}\lrincir{\frac{4\eta}{\mu}} \\ \nonumber
\le & \frac{ T \eta G^2}{\mu} + \frac{1}{\eta} \lrincir{\sum_{t=2}^T \lrincir{ B_{\Phi}(\y_t^\ast, \x_t) -  B_{\Phi}(\y_{t-1}^\ast, \x_t) } }  + \frac{1}{\eta} B_{\Phi}(\y_1^\ast, \x_1)  + \lrincir{\frac{\sigma}{2}}^{\frac{\sigma}{2-\sigma}}\lrincir{\frac{4\eta}{\mu}} \\ \nonumber
\refabovecir{\le}{\textcircled{1}} & \frac{ T \eta G^2}{\mu} + \frac{2G}{\eta}\sum_{t=1}^{T-1}\lrnorm{\y_{t+1}^\ast - \y_t^\ast}    + \frac{1}{\eta} B_{\Phi}(\y_1^\ast, \x_1)  + \lrincir{\frac{\sigma}{2}}^{\frac{\sigma}{2-\sigma}}\lrincir{\frac{4\eta}{\mu}} \\ \nonumber
\le & \frac{ T \eta G^2}{\mu} + \frac{2GD}{\eta}  + \frac{R^2}{\eta}  + \lrincir{\frac{\sigma}{2}}^{\frac{\sigma}{2-\sigma}}\lrincir{\frac{4\eta}{\mu}} \\ \nonumber
\lesssim  & \sqrt{TD}  + \sqrt{T}.
\end{align} $\textcircled{1}$ holds due to 
\begin{align}
\nonumber
B_\Phi(\y_{t+1}^\ast, \x_{t+1}) - B_{\Phi}(\y_t^\ast, \x_{t+1}) \le  2G \lrnorm{\y_{t+1}^\ast - \y_t^\ast}
\end{align} according to Lemma \ref{lemma_dynamic_regret_bound}. The last inequality holds by setting $ \eta = \min\left \{\sqrt{\frac{D+G}{T}}, \frac{\mu}{4} \right \}$.

Since it holds for any seqence of $f_t \in \Fcal$, we finally obtain
\begin{align}
\nonumber
\sup_{\{f_t\}_{t=1}^T \in \Fcal^T}\Rcal_D^{\textsc{MD-OCO}} \lesssim  \sqrt{TD}  + \sqrt{T}.
\end{align} 
It completes the proof.
\end{proof}

\textbf{Proof to Theorem \ref{theorem_lower_bound_oco}:}

\begin{proof}
This proof is inspired by \cite{Zhao:2018wx}, but our new analysis generalizes \cite{Zhao:2018wx} to the case of switching cost. 

Construct the function $f_t(\x_t) = \lrangle{\v_t, \x_t}$ for any $t\in[T]$. Here, $\v_t\in \{1, -1\}^d$, and every element $\v_t(j)$ with $j\in[d]$ is a random variable, which is sampled from a Rademacher distribution independently. For any online method $A\in\Acal$, its regret is bounded as follows.
\begin{align}
\nonumber
&\sup_{\{f_t\}_{t=1}^T} \Rcal_D^A \ge \Rcal_D^A  \\ \nonumber
= & \EE_{\v_{1:T}} \sum_{t=1}^T f_t(\x_t) + \sum_{t=1}^T \lrnorm{\x_t - \x_{t-1}}^{\sigma} - \EE_{\v_{1:T}} \min_{\{\y_t\}_{t=1}^T\in \Lcal_D^T}  \lrincir{\sum_{t=1}^T f_t(\y_t) + \sum_{t=1}^T \lrnorm{\y_t - \y_{t-1}}^{\sigma}} \\ \nonumber
= & \EE_{\v_{1:T}}\lrincir{ \sum_{t=1}^T f_t(\x_t) +  \sum_{t=1}^T \lrnorm{\x_t - \x_{t-1}}^{\sigma}} + \EE_{\v_{1:T}}  \max_{\{\y_t\}_{t=1}^T\in \Lcal_D^T}  \lrincir{ -\sum_{t=1}^T f_t(\y_t) -\sum_{t=1}^T \lrnorm{\y_t - \y_{t-1}}^{\sigma}  }    \\ \nonumber
= & \EE_{\v_{1:T}}\max_{\{\y_t\}_{t=1}^T\in \Lcal_D^T}  \sum_{t=1}^T \lrincir{ f_t(\x_t) -  f_t(\y_t) - \lrnorm{\y_t - \y_{t-1}}^{\sigma} } +  \EE_{\v_{1:T}}  \sum_{t=1}^T \lrnorm{\x_t - \x_{t-1}}^{\sigma}  \\ \label{equa_lower_bound_temp1}
= & \EE_{\v_{1:T}} \max_{\{\y_t\}_{t=1}^T\in \Lcal_D^T}  \sum_{t=1}^T \lrincir{ \lrangle{\v_t, \x_t - \y_t} - \lrnorm{\y_t - \y_{t-1}}^{\sigma} } +  \EE_{\v_{1:T}}  \sum_{t=1}^T \lrnorm{\x_t - \x_{t-1}}^{\sigma}.
\end{align}

For any optimal sequence of $\{\y_t^\ast\}_{t=1}^T$,
\begin{align}
\nonumber
\EE_{\v_t} \lrangle{\v_t, \x_{t-1}-\y_{t-1}^\ast} = \lrangle{\EE_{\v_t}\v_t, \x_{t-1}-\y_{t-1}^\ast} = \lrangle{\0, \x_{t-1}-\y_{t-1}^\ast} = 0.
\end{align} Thus, for any optimal sequence of $\{\y_t^\ast\}_{t=1}^T$, we have
\begin{align}
\nonumber
& \EE_{\v_{1:T}} \max_{\{\y_t\}_{t=1}^T\in \Lcal_D^T}  \sum_{t=1}^T \lrincir{ \lrangle{\v_t, \x_t - \y_t} -\lrnorm{\y_t - \y_{t-1}}^{\sigma} } \\ \nonumber
= & \EE_{\v_{1:T}} \lrincir{ \sum_{t=1}^T \lrangle{\v_t, \x_t - \y_t^\ast} -\sum_{t=1}^T \lrnorm{\y_t^\ast - \y_{t-1}^\ast}^{\sigma} } \\ \nonumber
= & \EE_{\v_{1:T}} \sum_{t=1}^T \lrangle{\v_t, \x_t-\x_{t-1}+\y_{t-1}^\ast - \y_t^\ast} - \EE_{\v_{1:T}} \sum_{t=1}^T \lrnorm{\y_t - \y_{t-1}^\ast}^{\sigma} \\ \nonumber
= & \EE_{\v_{1:T}} \sum_{t=1}^T \lrangle{\v_t, \x_t-\x_{t-1}} + \EE_{\v_{1:T}} \lrincir{ \sum_{t=1}^T \lrangle{\v_t, \y_{t-1}^\ast - \y_t^\ast}  - \sum_{t=1}^T \lrnorm{\y_t - \y_{t-1}^\ast}^{\sigma}   } \\ \nonumber
= & \EE_{\v_{1:T}} \sum_{t=1}^T \lrangle{\v_t, \x_t-\x_{t-1}} + \EE_{\v_{1:T}} \max_{\{\y_t\}_{t=1}^T\in \Lcal_D^T} \lrincir{ \sum_{t=1}^T \lrangle{\v_t, \y_{t-1} - \y_t}  - \sum_{t=1}^T \lrnorm{\y_t - \y_{t-1}}^{\sigma}   } \\ \nonumber
\end{align}

Substituting it into \eqref{equa_lower_bound_temp1}, we have
\begin{align}
\nonumber
&\sup_{\{f_t\}_{t=1}^T} \Rcal_D^A  \\ \nonumber
\ge & \EE_{\v_{1:T}} \lrincir{ \sum_{t=1}^T \lrangle{\v_t, \x_t-\x_{t-1}}  + \sum_{t=1}^T \lrnorm{\x_t - \x_{t-1}}^{\sigma} } \\ \nonumber 
& + \EE_{\v_{1:T}} \max_{\{\y_t\}_{t=1}^T\in \Lcal_D^T}  \lrincir{ \sum_{t=1}^T \lrangle{\v_t, \y_{t-1} - \y_t} - \sum_{t=1}^T \lrnorm{\y_t - \y_{t-1}}^{\sigma}  } \\ \nonumber
\refabovecir{\ge}{\textcircled{1}} & \EE_{\v_{1:T}} \max_{\{\y_t\}_{t=1}^T\in \Lcal_D^T}  \lrincir{ \sum_{t=1}^T \lrangle{\v_t, \y_{t-1} - \y_t} - \sum_{t=1}^T \lrnorm{\y_t - \y_{t-1}}^{\sigma}  } \\ \nonumber
\ge & \EE_{\v_{1:T}} \max_{\{\y_t\}_{t=1}^T\in \Lcal_D^T}  \sum_{t=1}^T  \lrangle{\v_t, \y_{t-1} - \y_t} - \max_{\{\y_t\}_{t=1}^T\in \Lcal_D^T} \sum_{t=1}^T \lrnorm{\y_t - \y_{t-1}}^{\sigma}   \\ \nonumber
\refabovecir{\ge}{\textcircled{2}}  &  \EE_{\v_{1:T}} \max_{\{\y_t\}_{t=1}^T\in \Lcal_D^T}  \sum_{t=1}^T  \lrangle{\v_t, \y_{t-1} - \y_t} -D^{\sigma} \\ \nonumber
\refabovecir{=}{\textcircled{3}} &  \EE_{\v_{1:T}} \max_{\{\y_t\}_{t=1}^T\in \Lcal_D^T}  \sum_{t=1}^T  \lrangle{\v_t, - \y_t} -D^{\sigma} \\ \nonumber
\refabovecir{=}{\textcircled{4}} &  \EE_{\v_{1:T}} \max_{\{\y_t\}_{t=1}^T\in \Lcal_D^T}  \sum_{t=1}^T  \lrangle{\v_t, \y_t} -D^{\sigma}. 
\end{align} $\textcircled{1}$ holds due to
\begin{align}
\nonumber
\EE_{\v_{t}} \lrincir{ \lrangle{\v_t, \x_t-\x_{t-1}}  + \lrnorm{\x_t - \x_{t-1}}^{\sigma} } = \lrangle{\EE_{\v_{t}}\v_t, \x_t-\x_{t-1}}  + \lrnorm{\x_t - \x_{t-1}}^{\sigma} = \lrnorm{\x_t - \x_{t-1}}^{\sigma}  \ge 0.
\end{align} $\textcircled{2}$ holds because that, for any sequence $\{\y_t\}_{t=1}^T$, $\sum_{t=1}^T \lrnorm{\y_t - \y_{t-1}} \le D$. Thus,
\begin{align}
\nonumber
\max_{\{\y_t\}_{t=1}^T\in \Lcal_D^T} \sum_{t=1}^T \lrnorm{\y_t - \y_{t-1}}^{\sigma} \le \max_{\{\y_t\}_{t=1}^T\in \Lcal_D^T} \lrincir{\sum_{t=1}^T \lrnorm{\y_t - \y_{t-1}}}^{\sigma} \le D^{\sigma}. 
\end{align} $\textcircled{3}$ holds because that, for any vector $\y_{t-1}$, 
\begin{align}
\nonumber
\EE_{\v_t} \lrangle{\v_t, \y_{t-1}} =  \lrangle{\EE_{\v_t}\v_t, \y_{t-1}} = \lrangle{\0, \y_{t-1}} = 0.
\end{align} $\textcircled{4}$ holds because that the domain of $\v_t$ is symmetric.

Furthermore, we construct  a sequence $\{\y_t\}_{t=1}^T$ as follows.

\begin{enumerate}
    \item Evenly split $\{\y_t\}_{t=1}^{T}$ into two subsets: $\{\y_t\}_{t=1}^{T_1}$ and $\{\y_{T_1+t}\}_{t=1}^{T_2}$. Here, $T_1 = T_2 = \frac{T}{2}$.
    \item After that, evenly split $\{\y_t\}_{t=1}^{T_1}$ into $N := \min\left\{ \frac{D}{R},T_1 \right\} $ subsets, that is, $\{\y_t\}_{t=1}^{\frac{T_1}{N}}$, $\{\y_t\}_{t=\frac{T_1}{N} + 1}^{\frac{2T_1}{N}}$, $\{\y_t\}_{t=\frac{2T_1}{N} + 1}^{\frac{3T_1}{N}}$, ..., $\{\y_t\}_{t=\frac{(N-1)T_1}{N}+1}^{T_1}$. 
    \item For the $i$-th subset of the sequence $\{\y_t\}_{t=1}^{T_1}$, let the values in it be same, and denote it by $\u_i$ with $\lrnorm{\u_i} \le \frac{R}{2}$. For the whole sequence  $\{\y_{T_1+t}\}_{t=1}^{T_2}$, let all the values be same, namely $\u_N$.
    \item For the sequence of $\{\y_t\}_{t=1}^{T_1}$, elements in different subsets are different such that $\|\u_{i+1} - \u_i\| \leq \|\u_{i+1}\| + \|\u_i\| \leq R$. Thus, 
    \begin{align}
    \nonumber
    & \sum_{t=1}^{T-1} \|\y_{t+1} - \y_{t}\| = \sum_{t=1}^{T_1-1} \|\y_{t+1} - \y_{t}\| + \sum_{t=T_1}^{T} \|\y_{t+1} - \y_{t}\| \\ \nonumber
    = & \sum_{i=1}^{N-1} \|\u_{i+1} - \u_i\| + 0\\ \nonumber
    \le & (N-1)R \\ \nonumber
    \le & D.
    \end{align} The last inequality holds due to $(N-1)R \le D$.
     It implies that $\{\y_t\}_{t=1}^{T}$ under our construction is feasible.
\end{enumerate}

Then, we have
\begin{align}
\nonumber
& \EE_{\v_{1:T}} \max_{\{\y_t\}_{t=1}^T\in \Lcal_D^T}  \sum_{t=1}^T  \lrangle{\v_t,  \y_t} = \EE_{\v_{1:T}} \max_{\{\y_t\}_{t=1}^T\in \Lcal_D^T}  \lrincir{ \sum_{t=1}^{T_1}  \lrangle{\v_t,  \y_t} + \sum_{t=T_1+1}^T  \lrangle{\v_t,  \y_t} } \\ \nonumber
= & \EE_{\v_{1:T}}  \sum_{i=1}^N  \max_{\lrnorm{\u_i}\le \frac{R}{2}} \lrangle{\sum_{t=1 + \frac{T(i-1)}{N}}^{\frac{Ti}{N}}\v_t, \u_i} + \EE_{\v_{1:T}} \max_{\lrnorm{\u_N}\le \frac{R}{2}} \lrangle{\sum_{t=T_1+1}^{T}\v_t, \u_N} \\ \nonumber
\refabovecir{=}{\textcircled{1}} & \frac{R}{2}\EE_{\v_{1:T}} \sum_{i=1}^N   \lrnorm{ \sum_{t=1 + \frac{T(i-1)}{N}}^{\frac{Ti}{N}}\v_t}  + \frac{R}{2} \EE_{\v_{1:T}}\lrnorm{ \sum_{t=T_1+1}^{T}\v_t }  \\ \nonumber
\refabovecir{\ge}{\textcircled{2}} & \frac{R}{2\sqrt{d}}\EE_{\v_{1:T}} \sum_{i=1}^N  \sum_{j=1}^d \left | \sum_{t=1 + \frac{T(i-1)}{N}}^{\frac{Ti}{N}}\v_t(j) \right | + \frac{R}{2\sqrt{d}} \EE_{\v_{1:T}}\sum_{j=1}^d \left | \sum_{t=T_1+1}^{T}\v_t(j) \right | \\ \nonumber
\refabovecir{=}{\textcircled{3}} & \frac{\sqrt{d}NR}{2} \cdot  \Omegacal{\sqrt{\frac{T}{N}}}  + \frac{R\sqrt{d}}{2} \cdot \Omega\lrincir{\sqrt{\frac{T}{2}}} \\ \nonumber
= & \Omegacal{\sqrt{R}\sqrt{TNR} + \sqrt{T}} \\ \nonumber
\refabovecir{=}{\textcircled{4}} & \Omegacal{\sqrt{TD}+ \sqrt{T}}.
\end{align} $\textcircled{1}$ holds because that the maximum is obtained at the boundary of the domain. $\textcircled{2}$ holds because that, for any $\v\in\RR^d$, $\lrnorm{\v}_1 \le \sqrt{d} \lrnorm{\v}_2$. $\textcircled{3}$ holds due to a classic result \cite{Hazan2016Introduction}, that is,
\begin{align}
\nonumber
\EE_{\v_{1:T}} \left | \sum_{t=1 + \frac{T(i-1)}{N}}^{\frac{Ti}{N}}\v_t(j) \right | = \Omegacal{\sqrt{\frac{T}{N}}}.
\end{align} $\textcircled{4}$ holds due to $D - R \le NR \le D + R$, which implies that $NR \lesssim D$ holds for $D>0$.

Therefore, we obtain
\begin{align}
\nonumber
\sup_{\{f_t\}_{t=1}^T} \Rcal_D^A \ge \EE_{\v_{1:T}} \max_{\{\y_t\}_{t=1}^T\in \Xcal^T}  \sum_{t=1}^T  \lrangle{\v_t, \y_t} - D^{\sigma} = \Omegacal{\sqrt{TD} + \sqrt{T}}.
\end{align} The last equality holds because $D^{\sigma}$ is a constant, and it does not increase over $T$.

Since it holds for any online algorithm $A\in \Acal$, we finally have 
\begin{align}
\nonumber
\inf_{A\in\Acal} \sup_{\{f_{t}\}_{t=1}^T \in \Fcal^T}  = \Omegacal{\sqrt{TD} + \sqrt{T}}.
\end{align} It completes the proof.
\end{proof}

\end{document}